\documentclass[12pt]{article}

\usepackage{macros}
\usepackage{thmtools}

\usepackage[round]{natbib}

\setcitestyle{authoryear}

\begin{document}

\title{Outlier-Robust Optimal Transport:\\
Duality, Structure, and Statistical Analysis}
\author{Sloan Nietert\thanks{Cornell University}, Rachel Cummings\thanks{Columbia University}, and Ziv Goldfeld\footnotemark[1]}

\maketitle

\begin{abstract}
The Wasserstein distance, rooted in optimal transport (OT) theory, is a popular discrepancy measure between probability distributions with various applications to statistics and machine learning. Despite their rich structure and demonstrated utility, Wasserstein distances are sensitive to outliers in the considered distributions, which hinders applicability in practice. 
We propose a new outlier-robust Wasserstein distance $\smash{\RWp}$ which allows for $\eps$ outlier mass to be removed from each contaminated distribution. 
Under standard moment assumptions, $\smash{\RWp}$ is shown to 
achieve strong robust estimation guarantees under the Huber $\eps$-contamination model.
Our formulation of this robust distance amounts to a highly regular optimization problem that lends itself better for analysis compared to previously considered frameworks.
Leveraging this, we conduct a thorough theoretical study of $\smash{\RWp}$, encompassing robustness guarantees, characterization of optimal perturbations, regularity, duality, and statistical estimation.
In particular, by decoupling the optimization variables, we arrive at a simple dual form for $\RWp$ that can be implemented via an elementary modification to standard, duality-based OT solvers. 
We illustrate the virtues of our framework via applications to generative modeling with contaminated datasets.
\end{abstract}

\section{Introduction}

Discrepancy measures between probability distributions are a fundamental constituent of statistical inference, machine learning, and information theory.
Among many such measures, Wasserstein distances \citep{villani2003} have recently emerged as a tool of choice for many applications. Specifically, for $p \in [1,\infty)$ and a pair of probability measures $\mu,\nu$ on a metric space $(\cX,d)$, the $p$-Wasserstein distance between them is\footnote{For $p=\infty$, we set $\Winfty(\mu,\nu) \coloneqq \inf_{\pi \in \Pi(\mu,\nu)}\|d\|_{L^\infty(\pi)}$.}
\begin{equation*}
\Wp(\mu,\nu)\coloneqq\left(\inf_{\pi\in\Pi(\mu,\nu)}\int_{\cX\times\cX}d(x,y)^p\dd\pi(x,y)\right)^{1/p},
\end{equation*}
where $\Pi(\mu,\nu)$ is the set of couplings for $\mu$ and $\nu$. The popularity of these metrics stems from a myriad of desirable properties, including rich geometric structure, metrization of the weak topology, robustness to support mismatch, and a convenient dual form. Modern applications thereof include generative modeling \citep{arjovsky_wgan_2017, gulrajani2017improved, tolstikhin2018wasserstein}, domain adaptation \citep{courty2014domain,courty2016optimal}, and robust optimization \citep{mohajerin_robust_2018,blanchet2018optimal, gao2016distributionally}.

Despite their advantages, Wasserstein distances suffer from sensitivity to outliers due to the strict marginal constraints, as even a small outlier mass can contribute greatly to the distance.
This has inspired a recent line of work into outlier-robust OT \citep{balaji2020, mukherjee2021, khang2021}%
, which relaxes the marginal constraints in various ways.
These build upon the theory of unbalanced OT \citep{piccoli2014, chizat2018,chizat2018scaling,liero2018, schmitzer2019} that quantifies the cost-optimal way to transform one measure into another via a combination of mass variation and transportation. 
We propose a new framework for outlier-robust OT that arises as the solution to a principled robust approximation problem. We conduct an in-depth theoretical study of the proposed robust distance, encompassing formal robustness guarantees, duality, characterization of primal minimizers / dual maximizers, regularity, and empirical convergence rates.

\subsection{Contributions}
\vspace{-1mm}
We introduce and study the \emph{$\eps$-outlier-robust Wasserstein distance} defined by
\vspace{-1mm}
\begin{gather}
\begin{aligned}
\label{eq:RWp}
    \RWp(\mu,\nu) \coloneqq \inf_{\substack{\mu',\nu' \in \cM_+(\cX) \\ \mu' \leq \mu, \: \|\mu - \mu'\|_\tv \leq \eps \\
    \nu' \leq \nu, \: \|\nu - \nu'\|_\tv \leq \eps}} \Wp\left(\frac{\mu'}{\mu'(\cX)},\frac{\nu'}{\nu'(\cX)}\right),
    \raisetag{5mm}
\end{aligned}
\end{gather}
where $\mu'$ and $\nu'$ are positive measures, $\|\cdot\|_{\mathsf{TV}}$ is the total variation (TV) norm, and $\leq$ denotes setwise inequality when appropriate.
The minimization over $\mu',\nu'$ allows outliers occupying less than fraction $\eps$ of probability mass to be omitted from consideration, after which the perturbed measures are renormalized.
Compared to prior work employing TV constraints \citep{balaji2020,mukherjee2021}, our definition has several distinct features: (1) it is naturally derived  as a robust proxy for $\Wp$ under the Huber $\eps$-contamination model; (2) it can be reframed as an optimization problem over a highly regular domain; and, consequently, (3) it admits a simple and useful duality theory.

We show that when the clean distributions have bounded $q$th moments for $q > p$, $\smash{\RWp}$ nearly achieves the minimax optimal robust estimation risk of $\eps^{1/p - 1/q}$ under the Huber $\eps$-contamination model. %
Moreover, our dual formulation mirrors the classic Kantorovich dual with an added penalty proportional to the range of the potential function. 
This provides an elementary robustification technique which can be applied to any duality-based OT solver: one needs only to compute the argmin and argmax of the discriminative potentials over the batch samples, which can be done in conjunction with existing gradient evaluations.
We demonstrate this the utility of this procedure with experiments for generative modeling with contaminated datasets using Wasserstein generative adversarial networks (WGANs) \citep{arjovsky_wgan_2017}.

We also study structural properties of $\RWp$, characterizing the minimizers of \eqref{eq:RWp} and maximizers of its dual, describing the regularity of the problem in $\eps$, and drawing a connection between $\RWp$ and loss trimming \citep{shen19}. Finally, we study statistical aspects of $\smash{\RWp}$, examining both one- and two-sample empirical convergence rates and providing additional robustness guarantees. The derived empirical convergence rates are at least as fast as the regular $n^{-1/d}$ rate for standard $\Wp$; however, faster rates may be possible if only a small amount of high-dimensional mass is present.

\subsection{Related Work}
\label{subsec:related_work}
The robust Wasserstein distance\footnote{Despite calling $\RWp$ a distance, we remark that the metric structure of $\Wp$ is lost after robustification.} in \eqref{eq:RWp} is closely related to the notions considered in \cite{balaji2020} and \cite{mukherjee2021}. In \cite{balaji2020}, similar constraints are imposed with respect to (w.r.t.) general $f$-divergences, but the perturbed measures are restricted to probability distributions. This results in a more complex dual form (derived by invoking standard Kantorovich duality on the Wasserstein distance between perturbations) and requires optimization over a significantly larger domain.
In \cite{mukherjee2021}, robustness w.r.t.\ the TV distance is added via a regularization term in the objective. This leads to a simple modified primal problem but the corresponding dual requires optimizing over two potentials, even when $p=1$. Additionally, \cite{khang2021} and \cite{nath2020} consider robustness via Kullback-Leibler (KL) divergence and integral probability metric (IPM) regularization terms, respectively. The former focuses on Sinkhorn-based primal algorithms and the latter introduces a dual form that is distinct from ours and less compatible with existing duality-based OT computational methods. In \cite{staerman21}, a median of means approach is used to tackle the dual Kantorovich problem from a robust statistics perspective. None of these works provide minimax error bounds for robust estimation.

The robust OT literature is intimately related to unbalanced OT theory, which addresses transport problems between measures of different mass \citep{piccoli2014, chizat2018, liero2018, schmitzer2019, hanin92}. These formulations are reminiscent of the problem \eqref{eq:RWp} but with regularizers added to the objective (KL being the most studied) rather than incorporated as constraints. Sinkhorn-based primal algorithms \citep{chizat2018scaling} are the standard approach to computation, and these have recently been extended to large-scale machine learning problems via minibatch methods \citep{fatras21a}. \citet{fukunaga2021} introduces primal-based algorithms for semi-relaxed OT, where marginal constraints for a single measure are replaced with a regularizer in the objective.
Partial OT \citep{caffarelli2010,figalli2010}, where only a fraction of mass needs to be moved, is another related framework.
However, \cite{caffarelli2010} consider a different parameterization of the problem, arriving at a distinct dual, and \cite{figalli2010} is mostly restricted to quadratic costs with no discussion of duality. Recently, \citet{chapel2020} has explored partial OT for positive-unlabeled learning, but dual-based algorithms are not considered.

\paragraph{Notation and Preliminaries.}
Let $(\cX,d)$ be a complete, separable metric space, and denote the diameter of a set $A \subset \cX$ by $\diam(A) \coloneqq \sup_{x,y \in A} d(x,y)$.~Take $C_b(\cX)$ as the set of continuous, bounded real functions on $\cX$, and let $\cM(\cX)$ denote the set of signed Radon measures on $\cX$ equipped with the TV norm\footnote{This definition will be convenient but omits a factor of 1/2 often present in machine learning literature.} $\|\mu\|_\tv \coloneqq |\mu|(\cX)$. Let $\cM_+(\cX)$ denote the space of finite, positive Radon measures on $\cX$. The Lebesgue measure on $\R^d$ is designated by $\lambda$. For $\mu,\nu \in \cM_+(\cX)$ and $p \in [1,\infty]$, we consider the standard $L^p(\mu)$ space with norm $\|f\|_{L^p(\mu)} = \left(\int |f|^p \dd \mu \right)^{1/p}$, and we write $\mu \leq \nu$ when $\mu(B) \leq \nu(B)$ for every Borel set $B \subseteq \cX$.

Let $\cP(\cX) \subset M_+(\cX)$ denote the space of probability measures on $\cX$, and take $\cP_p(\cX) \coloneqq \{ \mu \in \cP(\cX) : \int d(x,x_0)^p \dd \mu(x) < \infty \}$ to be those with bounded $p$th moment. We write $\cP_\infty(\cX)$ for probability measures with bounded support.
Given $\mu,\nu \in \cP(\cX)$, let $\Pi(\mu,\nu)$ denote the set of their couplings, i.e., $\pi \in \cP(\cX \times \cX)$ such that $\pi(B \times \cX) = \mu(B)$ and $\pi(\cX \times B) = \nu(B)$, for every Borel set $B$. When $\cX = \R^d$, we write the covariance matrix for $\mu \in \cP_2(\cX)$ as $\Sigma_\mu \coloneqq \E[(X - \E[X])(X - \E[X])^\intercal]$ where $X \sim \mu$. For $f:\cX \to \R$, we define the range $\Range(f): = \sup_{x \in \cX} f(x) - \inf_{x \in \cX} f(x)$. We write $a \lor b = \max\{a,b\}$ and use $\lesssim, \gtrsim, \asymp$ to denote inequalities/equality up to absolute constants.

Recall that $\Wp(\mu,\nu) < \infty$, for any $\mu,\nu \in \cP_p(\cX)$. For any $p \in [1,\infty)$, Kantorovich duality states that
\begin{equation}\label{eq:Wp-duality}
    \Wp(\mu,\nu)^p= \sup_{f \in C_b(\cX)} \int f \dd\mu + \int f^c \dd\nu,
\end{equation}
where the $c$-transform $f^c:\cX \to \R$ is defined by $f^c(y) = \inf_{x \in \cX} d(x,y)^p - f(x)$ (with respect to the cost $c(x,y) = d(x,y)^p$).

\section{Robust Estimation of \texorpdfstring{$\Wp$}{Wp}}
\label{sec:robust}

Outlier-robust OT is designed to address the fact that $\Wp\big((1-\eps)\mu + \eps \delta_x, \nu\big)$ explodes as $d(x,x_0) \to \infty$, no matter how small $\eps$ might be. More generally, one considers an adversary that seeks to dramatically alter $\Wp$ by adding small pieces of mass to its arguments.
To formalize this, we fix $p \in [1,\infty)$ and consider the Huber $\eps$-contamination model popularized in robust statistics \citep{huber64}, where a base measure $\mu \in \cP(\cX)$ is perturbed to obtain a contaminated measure $\tilde{\mu}$ belonging to the ball
\begin{align}
\label{eq:huber-ball}
    \cB_\eps(\mu) \coloneqq (1-\eps)\mu + \eps\,\cP(\cX)= \big\{ (1-\eps) \mu + \eps \alpha : \alpha \in \cP(\cX) \big\}.
\end{align}
The goal is to obtain a robust proxy $\hat{W}:\cP(\cX)^2 \to \R$ which, for any clean distributions $\mu,\nu \in \cP(\cX)$ with contaminated versions $\tilde{\mu} \in \cB_\eps(\mu), \tilde{\nu} \in \cB_\eps(\nu)$, achieves low error $\big|\hat{W}(\tilde{\mu}, \tilde{\nu}) - \Wp(\mu,\nu)\big|$.
In general, this error can be unbounded, so we require that the base measures belong to some family $\cD$ capturing distributional assumptions, e.g., bounded moments of some order. For any $\eps \in [0,1]$ and $\cD \subseteq \cP(\cX)$, define the minimax robust estimation risk %
by %
\begin{equation*}
    R(\cD,\eps) \coloneqq \inf_{\hat{W}:\cP(\cX)^2 \to \R} \sup_{\substack{\mu,\nu \in \cD\\ \tilde{\mu} \in \cB_\eps(\mu)\\ \tilde{\nu} \in \cB_\eps(\nu)}} \big|\hat{W}(\tilde{\mu},\tilde{\nu}) - \Wp(\mu,\nu)\big|.
\end{equation*}

The following theorem characterizes this risk under standard moment assumptions. Moreover, we show that $\RWp$ achieves this risk up to an additional multiplicative error term.

\begin{theorem}[Robust estimation of $\Wp$]
\label{thm:population-limit-robustness}
Fix $q > p$ and let $\cD_q \coloneqq \{ \mu \in \cP_q(\cX) : \|d(\cdot,x)\|_{L^q(\mu)} \leq M \text{ for some $x \in \cX$}\}$ denote the family of distributions with centered $q$th moments uniformly bounded by an absolute constant $M$. Then, for $0 \leq \eps \leq 0.49$\footnote{This bound of $0.49$ can be substituted with any constant less than the information-theoretic limit of $1/2$.}, we have
\begin{equation}
    R(\cD_q,\eps) \lesssim M \eps^{1/p - 1/q},
\end{equation}
and this bound is tight so long as $\cX$ contains two points at distance $M\eps^{-1/q}$. If further $\eps \leq 0.33$, then for any $\mu,\nu \in \cD_q$ with corrupted versions $\tilde{\mu} \in \cB_\eps(\mu), \tilde{\nu} \in \cB_\eps(\nu)$, we have
\begin{equation*}
    |\RWp(\tilde{\mu},\tilde{\nu}) - \Wp(\mu,\nu)| \leq \left[1 - (1-3\eps)^{1/p}\right]\Wp(\mu,\nu) + O(M\eps^{1/p - 1/q}).
\end{equation*}
\end{theorem}

The distance assumption is quite mild and is satisfied, e.g., when $\cX$ has a path connected component with diameter at least $M\eps^{-1/q}$. The risk bounds follow by characterizing appropriate moduli of continuity, mirroring classic techniques of \cite{donoho88}. In particular, for any $\cD \subseteq \cP(\cX)$, we have
\begin{equation*}
    \sup_{\substack{\alpha,\beta \in \cD\\ \alpha \in \cB_\eps(\beta)}} \Wp(\alpha,\beta) \lesssim R(\cD,\eps) \lesssim \sup_{\substack{\alpha \in \cD, \,\beta \in \cP(\cX)\\ \alpha \in \cB_{2\eps}(\beta)}} \Wp(\alpha,\beta),
\end{equation*}
with the estimator achieving the upper bound returning $\hat{W}(\tilde{\mu},\tilde{\nu}) = \Wp(\hat{\mu},\hat{\nu})$ for any $\hat{\mu},\hat{\nu} \in \cD$ such that $\tilde{\mu} \in \cB_\eps(\mu)$ and $\tilde{\nu} \in \cB_\eps(\nu)$. To bound the larger modulus, we use that $\Wp(\alpha,\beta) \leq \|d(\cdot,x)\|_{L^p(\alpha)} + \|d(\cdot,x)\|_{L^p(\beta)}$ for any $\alpha,\beta \in \cP(\cX)$ and $x \in \cX$, connecting $\Wp$ to moment bounds, and apply a bound from robust mean estimation. To control the smaller modulus, we simply consider $\alpha = \delta_{x}$ and $\beta = (1-\eps)\delta_x + \eps \delta_y$ for any $x,y \in \cX$ such that $d(x,y) = M\eps^{-1/q}$. Finally, the risk bound for $\RWp$, which matches the minimax risk up to a multiplicative error term of at most $3\eps \Wp(\mu,\nu)$, relies on a lemma showing that $|\RWp(\tilde{\mu},\tilde{\nu}) - \Wp(\mu,\nu)| \leq |(1-3\eps)\Wp^{3\eps}(\mu,\nu) - \Wp(\mu,\nu)|$. Full details are provided in \cref{prf:robustness}.

\medskip  
We next specialize \cref{thm:population-limit-robustness} to the common case of $\cX = \R^d$ with base measures whose covariance matrices $\Sigma_\mu,\Sigma_\nu$ have bounded spectral norms.
Such measures also have bounded moments in the sense of \cref{thm:population-limit-robustness}, so the previous upper bound applies. The lower bound for this case (given in \cref{prf:robustness-cov}) uses the same technique as before but with a more careful choice of measures involving a multivariate Gaussian.

\begin{corollary}[Bounded covariance]
\label{cor:robustness-cov}
Fix $\cX = \R^d$ and let $\cD_2^{\mathrm{cov}} \coloneqq \{ \mu \in \cP_2(\cX): \Sigma_\mu \preceq I_d \}$. For $p < 2$ and $0 \leq \eps \leq 0.49$, we have $R(\cD_2^{\mathrm{cov}}, \eps) \asymp \sqrt{d} \, \eps^{1/p-1/2}$. When $\eps \leq 0.33$, this risk is achieved by $\hat{W} = \RWp$ up to an additional term of $O(\eps \,\Wp(\mu,\nu))$, as in \cref{thm:population-limit-robustness}.
\end{corollary}

\begin{remark}[Comparison with robust mean estimation]
    In the setting of mean estimation under assumptions analogous to \cref{cor:robustness-cov} (i.e., $\Sigma_\mu \preceq I_d$, $\tilde{\mu} \in \cB_\eps(\mu)$), the optimal error rate of $\sqrt{\eps}$ is \emph{dimension-free} \citep{mengjie2018}. We interpret the factor of $\sqrt{d}$ present for our $\Wp$ rate as reflecting the high-dimensional optimization inherent to the Wasserstein distance.
\end{remark}

\begin{remark}[Asymmetric contamination]
The robust distance from \eqref{eq:RWp} readily extends to an asymmetric distance $\Wp^{\eps_\mu,\eps_\nu}$ with distinct robustness radii, so that $\RWp = \Wp^{\eps,\eps}$. Extensions of our main results (including \cref{thm:population-limit-robustness}) to this setting are presented in \cref{app:asymmetric-results}. The one-sided version $\RWp(\mu \| \nu) \coloneqq \Wp^{\eps,0}(\mu,\nu)$ is well-suited for applications such as generative modeling (see \cref{SEC:applications}).
\end{remark}

Our precise error bounds exploit the unique structure of $\RWp$ and do not translate clearly to existing robust proxies for $\Wp$.
We note, however, that the TV-robustified $\Wp$ presented in \citet{balaji2020} can be controlled and approximated to some extent by $\RWp$, via bounds presented in \cref{app:equivalent-formulations}. 

\section{Duality Theory for \texorpdfstring{$\RWp$}{Robust Wp}}
\label{sec:duality}

In addition to its robustness properties, $\RWp$ enjoys a simple optimization structure that enables a useful duality theory. Unless stated otherwise, we henceforth assume that $\cX$ is compact. To begin, we reformulate $\RWp(\mu,\nu)$ as a minimization problem over Huber balls centered at $\mu$ and $\nu$.

\begin{restatable}[Mass addition]{proposition}{massaddition}
\label{prop:mass-addition}
For all $p \in [1,\infty]$ and $\mu,\nu \in \cP(\cX)$, we have
\begin{align}
\label{eq:mass-addition}
    \RWp(\mu,\nu)= (1-2\delta)^{-1/p} \hspace{-2mm} \inf_{\substack{\mu' \in \cB_{\delta}(\mu)\\ \nu' \in \cB_{\delta}(\nu)}} \Wp(\mu',\nu'),
\end{align}
where $\delta = \eps/(1+\eps)$.
\end{restatable}

While the original definition \eqref{eq:RWp} involves \emph{removing} mass from the base measures and rescaling, \eqref{eq:mass-addition} is optimizing over mass \emph{added} to $\mu$ and $\nu$ (up to scaling). Our proof in \cref{prf:mass-addition} of this somewhat surprising result relies on the symmetric nature of the OT distance objective. Roughly, instead of removing a piece mass from one measure, we may always add it to the other.

This reformulation is valuable because the updated constraint sets are simple and do not interact with the simplex boundary. Specifically, definition \eqref{eq:huber-ball} reveals that the Huber ball $\cB_\delta(\mu)$ is always an affine shift of $\cP(\cX)$, with scaling independent of $\mu$. Hence, linear optimization over $\cB_\delta(\mu)$ is straightforward:
\begin{equation}
\label{eq:decoupling}
    \inf_{\mu' \in \cB_\delta(\mu)} \int f \dd \mu' = (1-\eps)\int f \dd \mu + \delta \inf_{x \in \cX} f(x).
\end{equation}

The above stands in contrast to the TV balls (i.e., sets of the form $\{ \mu' \in \cP(\cX): \|\mu' - \mu\|_\tv \leq \delta \}$) that appear in existing robust OT formulations; these exhibit non-trivial boundary interactions as depicted in \cref{fig:huber-vs-tv-ball}.
Fortunately, $\Wp$ is closely tied to the linear form $\mu \mapsto \int f \dd \mu$ via Kantorovich duality---a cornerstone for various theoretical derivations and practical implementations. Combining this with a minimax result, we establish a related dual form for $\RWp$.

\begin{figure}[t]
 \centering
 \scalebox{0.3}
 {\includegraphics{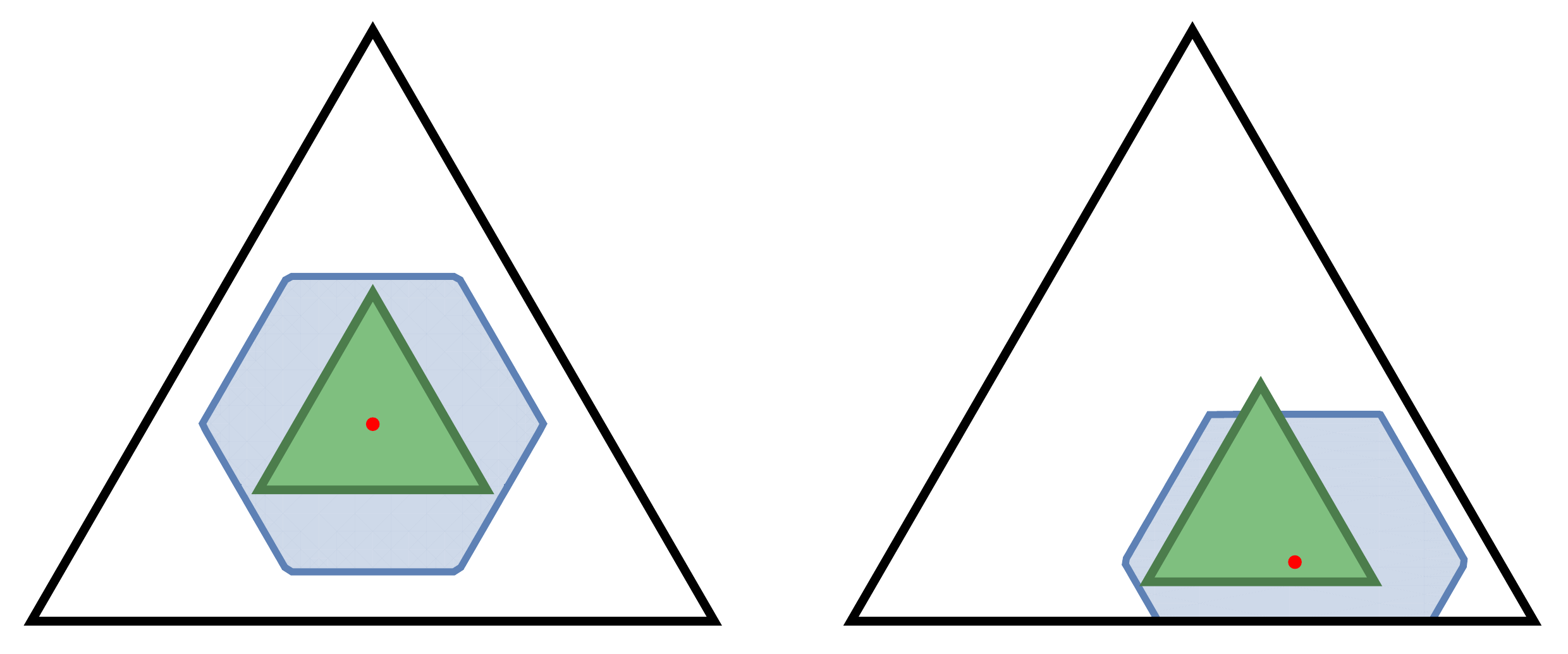}}
 \caption{Huber $\eps$-contamination balls (green) and $\eps$-TV balls (blue), centered at distinct points (red) within the 2-dimensional simplex. The Huber balls with different centers are translates of each other, while the rightmost TV ball interacts non-trivially with the simplex boundary.} \label{fig:huber-vs-tv-ball}
\end{figure}

\begin{restatable}[Dual form]{theorem}{RWpdual}
\label{thm:RWp-dual}
For $p \in [1,\infty)$, $\eps \in [0,1]$, and $\mu,\nu \in \cP(\cX)$, we have
\begin{align}
\label{eq:RWp-dual}
\begin{split}
    (1 - \eps) \RWp(\mu,\nu)^p &= \sup_{\substack{f \in C_b(\cX)}} \int f \dd \mu + \int f^c \dd \nu -  2 \eps \|f\|_\infty\\
    &= \sup_{\substack{f \in C_b(\cX)}} \int f \dd \mu + \int f^c \dd \nu - \eps \Range(f),\raisetag{7.5mm}
\end{split}
\end{align}
and the suprema are achieved by $f \in C_b(\cX)$ with $f = (f^c)^c$.
\end{restatable}

This new formulation differs from the classic dual \eqref{eq:Wp-duality} by a range penalty for the potential function. When $p=1$, we have $f^c = -f$ and $f = (f^c)^c$ exactly when $f$ is 1-Lipschitz.
The theorem is proven in \cref{prf:RWp-dual}, where we first apply \cref{prop:mass-addition} and then invoke Kantorovich duality for $\Wp(\mu',\nu')$, while verifying that the conditions for Sion's minimax theorem hold true. Applying minimax gives
\begin{align*}
    \inf_{\substack{\mu' \in \cB_{\delta}(\mu)\\\nu' \in \cB_{\delta}(\nu)}} \Wp(\mu',\nu') =\, \sup_{f \in C_b(\cX)} \left[ \inf_{\mu' \in \cB_{\delta}(\mu)} \int f \dd \mu' + \inf_{\nu' \in \cB_{\delta}(\nu)} \int f^c \dd \nu' \right],
\end{align*}
where $\delta = \eps/(1+\eps)$.
At this point, we employ \eqref{eq:decoupling}, along with properties of the $c$-transform, to obtain the desired dual. We stress that if $\mu'$ and $\nu'$ instead varied within TV balls, the inner minimization problems would not admit closed forms due to boundary interactions. 

\cref{thm:RWp-dual} reveals an elementary procedure for robustifying the Wasserstein distance against outliers: regularize its standard Kantorovich dual w.r.t.\ the sup-norm of the potential function. The simplicity of this modification is its main strength. As demonstrated in \cref{SEC:applications}, this enables adjusting popular duality-based OT solvers, e.g., \citet{arjovsky_wgan_2017}, to the robust framework and opens the door for applications to generative modeling with contaminated datasets. We provide an interpretation for the maximizing potentials in \cref{sec:structure}. Some concrete examples for computing $\RWp$ are found in \cref{app:examples}.

\begin{remark}[TV as a dual norm]
\label{rem:dual-norm}
Recall that $\|\cdot\|_\tv$ is the dual norm corresponding to the Banach space of measurable functions on $\cX$ equipped with $\|\cdot\|_\infty$. An inspection of the proof of \cref{thm:RWp-dual} reveals that our penalty scales with $\|\cdot\|_\infty$ precisely for this reason.
\end{remark}

Finally, we describe an alternative dual form which ties robust OT to loss trimming---a popular practical tool for robustifying estimation algorithms when $\mu$ and $\nu$ have finite support \citep{shen19}.

\begin{restatable}[Loss trimming dual]{proposition}{losstrimming}
\label{prop:loss-trimming}
Fix $p \in [1,\infty)$ and take $\mu,\nu$ to be uniform distributions over $n$ points each. If $\eps \in [0,1]$ is a multiple of $1/n$, then 
\begin{align*}
\RWp(\mu,\nu)^p=\sup_{f \in C_b(\cX)}\left( \min_{\substack{\cA \subseteq \supp(\mu)\\|\cA| = (1 - \eps)n}} \frac{1}{|\cA|}\sum_{x \in \cA} f(x) + \min_{\substack{\cB \subseteq \supp(\nu)\\|\cB| = (1 - \eps)n}} \frac{1}{|\cB|}\sum_{y \in \cB} f^c(y) \right).
\end{align*}
\end{restatable}

The inner minimization problems above clip out the $\eps n$ fraction of samples whose potential evaluations are largest. This is similar to how standard loss trimming clips out a fraction of samples that contribute most to the considered training loss.

\section{Structural Properties}
\label{sec:structure}

We turn to structural properties of $\RWp$, exploring primal and dual optimizers, regularity of $\RWp$ in $\epsilon$, as well as an alternative (near) coupling-based primal form. 

\subsection{Primal and Dual Optimizers}

We first prove that there are primal and dual optimizers satisfy certain regularity conditions.

\begin{restatable}[Existence of minimizers]{proposition}{minimizers}
\label{prop:minimizers}
For $p \in [1,\infty]$ and $\mu,\nu \in \cP(\cX)$, the infimum in \eqref{eq:RWp} is achieved, and there are minimizers $\mu' \leq \mu$ and $\nu' \leq \nu$ such that $\mu',\nu' \geq \mu \land \nu$ and $\mu'(\cX) = \nu'(\cX) = 1 - \eps$.
\end{restatable}

\begin{figure}[b]
 \centering
 \scalebox{0.7}
 {\includegraphics{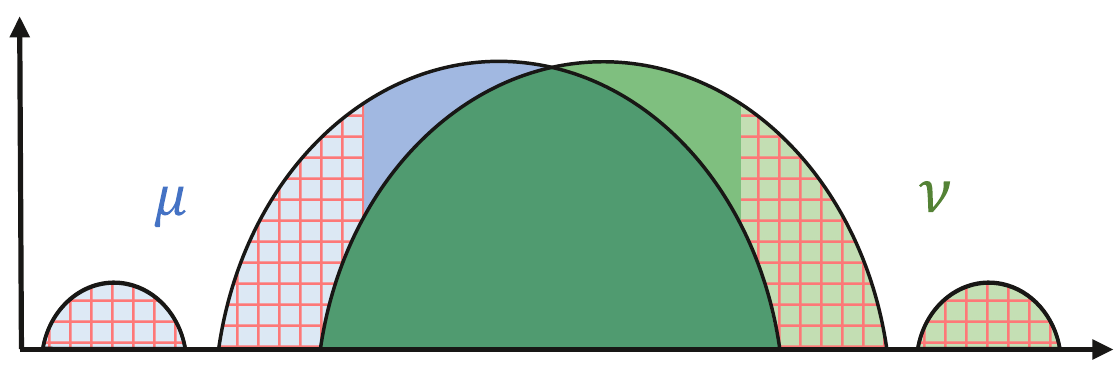}}
 \caption{The gridded light blue and green regions each have mass $\eps$, respectively, and are removed to obtain optimal $\mu'$ and $\nu'$ for $\RWone$. No mass need be removed from the dark region designating $\mu \land \nu$.} \label{fig:primal-perturbations}
\end{figure}

We remark that the lower envelope of $\mu \land \nu$, illustrated in \cref{fig:primal-perturbations}, is straightforward when $p=1$, since $\Wone(\mu,\nu)$ is a function of $\mu - \nu$. However, this conclusion is not obvious for $p > 1$, and its proof in \cref{prf:minimizers-existence} utilizes a discretization argument. For achieving the infimum, we show for $p < \infty$ that the constraint set is compact w.r.t.\ the classic Wasserstein topology, while the objective is clearly continuous in $\Wp$. For $p = \infty$, we observe that the constraint set is compact and the objective is lower semicontinuous w.r.t.\ the topology of weak convergence over $\cP_\infty(\cX)$.

\begin{restatable}[Interpreting maximizers]{proposition}{RWpdualmaximizers}
\label{prop:RWp-dual-maximizers}
If $f \in C_b(\cX)$ maximizes \eqref{eq:RWp-dual}, then any $\mu',\nu'\in\cM_+(\cX)$ minimizing \eqref{eq:RWp} satisfy $\supp(\mu-\mu') \subseteq \argmax(f)$ and $\supp(\nu-\nu') \subseteq \argmin(f)$.
\end{restatable}

{
\floatsetup[figure]{capposition={bottom}}
\begin{figure}[t!]
\begin{center}
\includegraphics[width=0.45\linewidth]{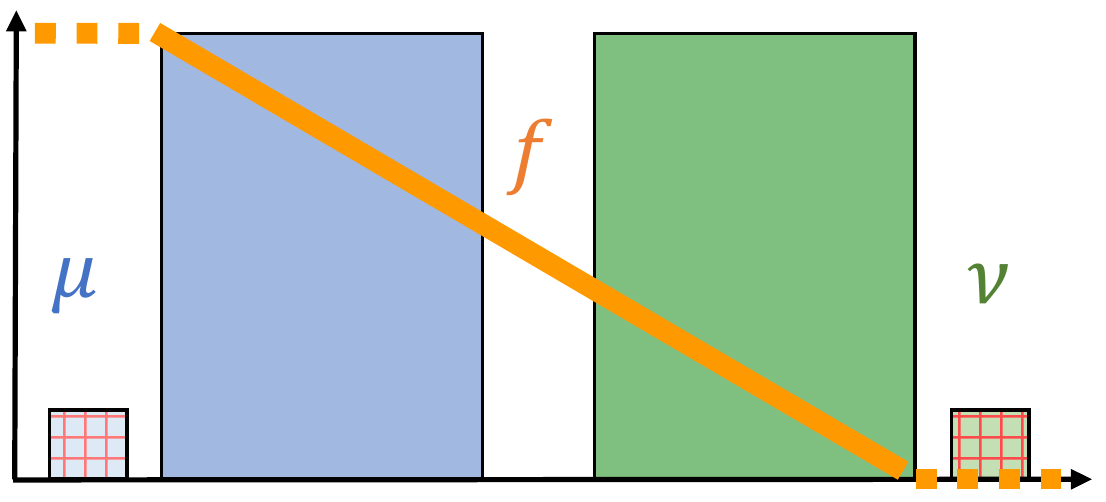}
\:\:\:
\includegraphics[width=0.45\linewidth]{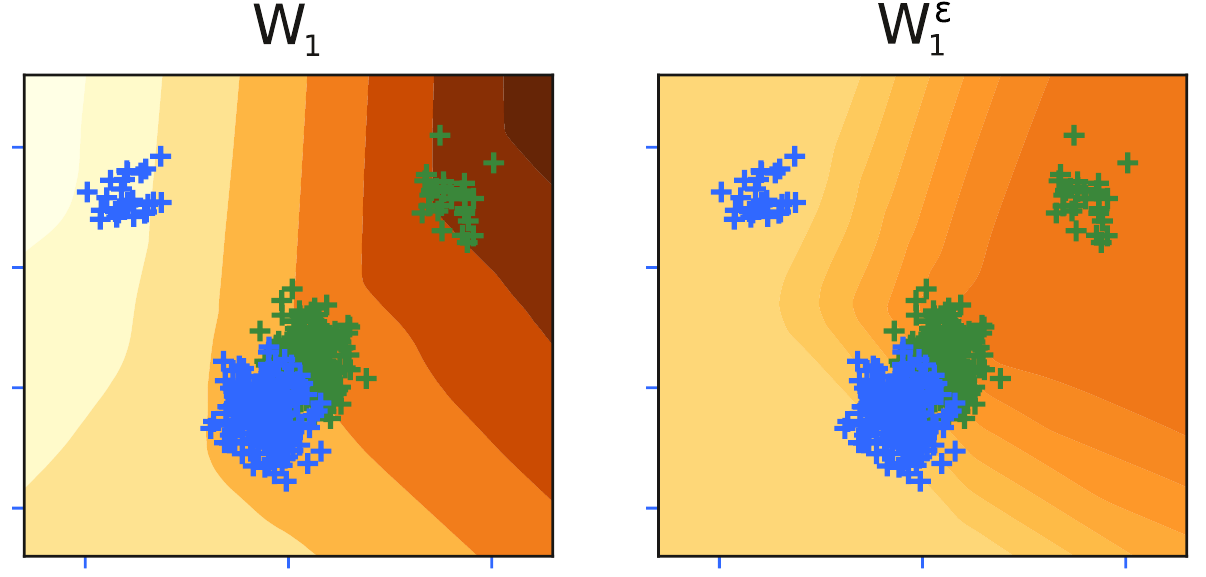}
\end{center}
\vspace{-5mm}
\caption{Optimal potentials: (left) 1D densities plotted with their optimal potential for the $\RWone$ dual problem; (right) contour plots for optimal dual potentials to $\Wone$ and $\RWone$ between 2D Gaussian mixtures. Observe how optimal potentials for the robust dual are flat over outlier mass.}\label{fig:optimal-potentials}
\end{figure}
}

Thus, the level sets of the dual potential encode the location of outliers in the original measures, as depicted in \cref{fig:optimal-potentials}. In fact, optimal perturbations $\mu-\mu'$ and $\nu-\nu'$ are sometimes determined exactly by an optimal potential~$f$, often taking the form $\mu|_{\argmax(f)}$ and $\nu|_{\argmax(f)}$ (though not always; we discuss this in \cref{prf:RWp-dual-maximizers} along with the proof). %

\subsection{Regularity in Robustness Radius}

We examine how $\RWp$ depends on the robustness~radius.

\begin{restatable}[Dependence on $\eps$]{proposition}{RWpepsdependence}
\label{prop:RWp-eps-dependence}
For any $p\in[1,\infty]$, $0\leq\eps\leq\eps'\leq1$, and $\mu,\nu \in \cP(\cX)$, we have 
\begin{enumerate}[(i)]
    \item $\Wp^{\|\mu-\nu\|_\tv/2}(\mu,\nu) = 0$, $\mathsf{W}_p^0(\mu,\nu) = \mathsf{W}_p(\mu,\nu)$;
    \item $\Wp^{\eps'}(\mu,\nu)\leq \Wp^{\eps}(\mu,\nu)$;
    \item $\Wp^{\eps}(\mu,\nu)\leq\left(\frac{1-\eps'}{1-\eps}\right)^{\frac{1}{p}}\Wp^{\eps'}(\mu,\nu)+ 4\diam(\cX)\left(\frac{\eps' - \eps}{1-\eps}\right)^{\frac{1}{p}}.$
\end{enumerate}
\end{restatable}
The proof is given in \cref{prf:RWp-eps-dependence}. More precise, diameter independent, bounds in the form of (iv) are provided in the proofs of the robustness results from \cref{sec:robust}, but these require $\mu$ and $\nu$ to satisfy certain moment bounds that do not hold in general.

\subsection{Alternative Primal Form}
    
Mirroring the primal Kantorovich problem for classic $\Wp$, we derive in \cref{prf:RWp-coupling-primal} an alternative primal form for $\RWp$ in terms of (near) couplings for $\mu$ and $\nu$.

\begin{restatable}[Alternative primal form]{proposition}{RWpcouplingprimal}
\label{prop:RWp-coupling-primal}
For any $p \in [1,\infty]$ and $\mu,\nu \in \cP(\cX)$, we have
\begin{equation}
\label{eq:RWp-coupling-primal}
    \RWp(\mu,\nu) = \inf_{\substack{\pi \in \cP(\cX \times \cX)\\ \mu \in \cB_\eps(\pi_1), \, \nu \in \cB_\eps(\pi_2)}} \|d\|_{L^p(\pi)},
\end{equation}
where $\pi_1$ and $\pi_2$ are the respective marginals of $\pi$.
\end{restatable}

\begin{remark}[Data privacy]
From this, we deduce that $\RWinfty(\mu,\nu) \leq M$ if and only if there exists a coupling $(X,Y)$ for $\mu,\nu$ such that $|X-Y| \leq M$ with probability at least $1 - \Theta(\eps)$.
In \cref{app:privacy}, we state this more precisely and leverage this fact for an application to data privacy. Specifically, within the Pufferfish privacy framework, the so-called Wasserstein Mechanism \citep{song2017} maximizes $\Winfty$ over certain pairs of distributions to provide a strong privacy guarantee. By substituting $\Winfty$ with $\RWinfty$, we reach an alternative mechanism that satisfies a slightly relaxed guarantee and can introduce significantly less noise.
\end{remark}

\cref{prop:RWp-eps-dependence} implies that $\lim_{\eps\to 0} \RWp =\Wp$, posing $\RWp$ as a natural extension of $\Wp$. Given the representation in \cref{prop:RWp-coupling-primal}, we now ask whether convergence of optimal (near) couplings also holds. A proof in \cref{prf:coupling-delta-dependence} provides an affirmative answer via a $\Gamma$-convergence~argument.

\begin{restatable}[Convergence of couplings]{proposition}{couplingdeltadependence}
\label{prop:coupling-delta-dependence}
Fix $p \in [1,\infty]$ and $\mu,\nu \in \cP(\cX)$. If $\eps_n \searrow 0$ and $\pi_n \in \cP(\cX \times \cX)$ is optimal for $\Wp^{\eps_n}(\mu,\nu)$ via \eqref{eq:RWp-coupling-primal}, for each $n \in \N$, then $\{\pi_n\}_{n \in \N}$ admits a subsequence converging weakly to an optimal coupling for $\Wp(\mu,\nu)$.
\end{restatable}

Finally, we consider a case of practical importance: the discrete setting where $\mu$ and $\nu$ have finite supports. Like classic OT, computing $\RWp(\mu,\nu)$ between discrete measures amounts to a linear program for $p < \infty$ and can be solved in polynomial time. The proof in \cref{prf:LP} starts from the alternative primal form of \cref{prop:RWp-coupling-primal} and analyzes the feasible polytope when the support sizes are equal.

\begin{restatable}[Finite support]{proposition}{lp}
\label{prop:lp}
Let $\mu$ and $\nu$ be uniform discrete measures over $n$ points each. Then there exist optimal $\mu',\nu'$ for $\RWp(\mu,\nu)$ such that $\mu'$ and $\nu'$ each give mass $1/n$ to $\lfloor (1 - \eps) n \rfloor$ points and assign their remaining $\lceil \eps n \rceil / n - \eps$ mass to a single point.
\end{restatable}
When $\eps$ is a multiple of $1/n$, the propositions says that there exist minimizers which assign equal mass to $(1-\eps)n$ points, while eliminating the remaining $\eps n$ that are identified as outliers.

\section{Statistical Analysis}

We now examine estimation of $\RWp$ from observed data, fixing $\eps \leq 1/2$ throughout.

\subsection{Empirical Convergence Rates}

In practice, we often have access only to samples from $\mu,\nu \in \cP(\cX)$, which motivates the study of empirical convergence under $\RWp$. Consider the empirical measures $\hat{\mu}_n\coloneqq n^{-1}\sum_{i=1}^n \delta_{X_i}$ and $\hat{\nu}_n\coloneqq n^{-1}\sum_{i=1}^n \delta_{Y_i}$, where $X_1, \dots, X_n$ and $Y_1, \dots, Y_n$ are i.i.d.\ samples from $\mu$ and $\nu$, respectively. We examine both the one- and two-sample scenarios, i.e., the speed at which $\RWp(\mu,\hat{\mu}_n)$ and $\big|\RWp(\hat{\mu}_n,\hat{\nu}_n) - \RWp(\mu,\nu)\big|$ converge to 0 as $n$ grows. 

This section assumes that $\cX \subset \R^d$ with $d > 2$; extensions to the non-Euclidean case are provided in the \crefrange{prf:one-sample-rates}{prf:two-sample-rate}. To state the results, we need some definitions. 
Recall the covering number $\cN_\delta(\gamma,\tau)$, defined for $\delta > 0, \tau \geq 0$, and $\gamma \in \cM_+(\cX)$ as the minimum number of closed balls of radius $\delta$ needed to cover $\cX$ up to a set $A$ with $\gamma(A) \leq \tau$. We define the lower $\tau$-covering dimension for $\gamma$ as
\[
    d_*^\tau(\mu) \coloneqq \liminf_{\delta \to 0} \frac{\cN_\delta(\gamma,\tau)}{-\log(\delta)}.
\]
We note that lower bounds for standard $\Wp$ depend on the lower Wasserstein dimension given by $\lim_{\tau \to 0} d_*^\tau(\mu)$ \citep{weed2019}. To provide meaningful bounds for $\RWp$, on the other hand, we require control of $d_*^\tau(\mu)$ for some $\tau > \eps$, which can be understood as a robust notion of dimension.

\begin{restatable}[One-sample rates]{proposition}{onesamplerates}
\label{prop:one-sample-rates}
Fix $p \in [1,\infty]$ and $\tau \in (\eps,1]$. If $p < d/2$, we have
\begin{equation*}
    \E\left[\RWp(\mu,\hat{\mu}_n)\right] \leq C n^{-1/d}
\end{equation*} 
for a constant $C$ independent of $n$ and $\eps$.
Furthermore, if $d_*^\tau(\mu) > s$ , then
\begin{equation*}
    \RWp(\mu,\gamma) \geq C'\left(\tau - \eps\right)^{1/p} n^{-1/s}
\end{equation*}
for any $\gamma$ supported on at most $n$ points, including $\hat{\mu}_n$, where $C' > 0$ is an absolute constant.
\end{restatable}

The robust $\tau$-covering dimension can be smaller than standard notions of intrinsic dimension when all but the outlier mass is lower dimensional. Our lower bound captures this fact. The upper bound depends on the ambient dimension $d$ since there is no guarantee that only outlier mass is supported on a high-dimensional set\footnote{Our upper bound holds if we substitute $d$ with anything greater than upper Wasserstein dimension---another notion of intrinsic dimensionality defined in \cite{weed2019}.}. Indeed, the following corollary shows that when a significant portion of mass is high-dimensional, the $n^{-1/d}$ rate is sharp. 

\begin{restatable}[Simple one-sample rate]{corollary}{simpleonesamplerate}
\label{cor:simple-one-sample-rate}
Fix $p < d/2$ and $\eps \in (0,1]$. If $\mu$ has absolutely continuous part $f \dd \lambda$ with $\int f \dd \lambda > \eps$ and $f$ bounded from above, then $d_*^\tau(\mu) = d$ and $\E\left[\RWp(\mu,\hat{\mu}_n)\right] = \Theta(n^{-1/d})$.
\end{restatable}
In words, if more than the $\eps$ mass that can be removed via robustification is absolutely continuous (w.r.t. Lebesgue on $\RR^d$), then the standard $n^{-1/d}$ ``curse of dimensionality'' applies. Nevertheless, we conjecture that an upper bound that depends on a robust upper dimension can be derived under appropriate assumptions on $\mu$, as discussed in \cref{SEC:summary}. The proofs of \cref{prop:one-sample-rates} and \cref{cor:simple-one-sample-rate} are found in Appendices \ref{prf:one-sample-rates} and \ref{prf:simple-one-sample-rate}, respectively.

Moving to the two-sample regime, we again have an upper bound which matches the standard rate for $\Wp$.
\begin{restatable}[Two-sample rate]{proposition}{twosamplerate}
\label{prop:two-sample-rate}
For any $\eps \in [0,1]$ and $p < d/2$, we have
\begin{equation*}
    \E\left[\left|\RWp(\hat{\mu}_n,\hat{\nu}_n)^p - \RWp(\mu,\nu)^p\right|\right] \leq C n^{-p/d}.
\end{equation*}
for a constant $C$ independent of $n$ and $\eps$.
\end{restatable}
The proof in \cref{prf:two-sample-rate} is a consequence of the dual form from \cref{thm:RWp-dual}, combined with standard one-sample rates for $\Wp$. There, we also discuss obstacles to extending two-sample lower bounds for standard $\Wp$ to the robust setting. 

\subsection{Additional Robustness Guarantees}

Finally, we provide conditions under which $\Wp(\mu,\nu)$ can be recovered precisely from $\RWp(\mu,\nu)$ despite data contamination. Naturally, these conditions are stronger than those needed for approximate (minimax optimal) robust approximation, as studied in \cref{sec:robust}.
We return to the Huber $\eps$-contamination model, taking $\tilde{\mu} \in \cB_\eps(\mu)$ and $\tilde{\nu} \in \cB_\eps(\nu)$. One cannot hope to achieve exact recovery via $\RWp$ in general, since $\RWp(\tilde{\mu},\tilde{\nu}) < \Wp(\mu,\nu)$ when $\tilde{\mu} = \mu$ and $\tilde{\nu} = \nu$ for $p < \infty$. Nevertheless, we can provide exact recovery guarantees under appropriate mass separation assumptions.

\begin{restatable}[Exact recovery]{proposition}{exactrecovery}
\label{prop:exact-recovery}
Fix $\mu,\nu \in \cP(\cX)$, $\eps \in [0,1]$, and suppose that $\tilde{\mu} = (1-\eps)\mu + \eps \alpha$ and $\tilde{\nu} = (1-\eps)\nu + \eps \beta$, for some $\alpha,\beta \in \cP(\cX)$. Let $S = \supp(\mu + \nu)$. If $d(\supp(\alpha),S)$, $d(\supp(\beta),S)$, and $d(\supp(\alpha),\supp(\beta))$ are all greater than $\diam(S)$, then $\RWp(\tilde{\mu},\tilde{\nu}) = \Wp(\mu,\nu)$.
\end{restatable}
Our proof in \cref{prf:exact-recovery} uses an infinitesimal perturbation argument and shows that, when outliers are sufficiently far away, removing outlier mass is strictly better than inlier mass for minimizing $\Wp$. The appendix also discussed more flexible albeit technical assumptions under which exact recovery is possible, and provides bounds for when the robustness radius does not match the contamination level $\eps$ exactly.

\cref{prop:exact-recovery} relies on the contamination level $\eps$ being known, which is often not the case in practice. To account for this, we prove in \cref{prf:radius} that, under the same assumptions, there exists a principled approach for selecting the robustness radius when $\eps$ is unknown.

\begin{restatable}[Robustness radius for unknown~$\eps$]{proposition}{radius}
\label{prop:radius}
Assume the setting of \cref{prop:exact-recovery} and let $D$ be the maximum of $d(\supp(\alpha),S)$, $d(\supp(\beta),S)$, and $d(\supp(\alpha),\supp(\beta))$. Then, for $p \in [1,\infty)$, we have
\begin{align*}
    & \textstyle{\dv{\delta}}\! \left[(1\!-\!\delta)\Wp^{\delta}(\tilde{\mu},\tilde{\nu})^p\right] \leq -D^p, & \delta \in [0,\eps)\\
    & \textstyle{\dv{\delta}}\!\left[(1\!-\!\delta)\Wp^{\delta}(\tilde{\mu},\tilde{\nu})^p\right] \geq -\diam(S)^p > -D^p,&\delta \in (\eps,1]
\end{align*}
at the (all but countably many) points where the derivative is defined.
\end{restatable}

As $(1-\delta)\Wp^{\delta}(\tilde{\mu},\tilde{\nu})^p$ is continuous and decreasing in $\delta$ by \cref{prop:RWp-eps-dependence}, we have identified an ``elbow'' in this curve located exactly at the true contamination level~$\eps$. 

Next, we return to the statistical setting and show how to obtain exact recovery assuming that the fraction of corrupted samples vanishes as $n \to \infty$.
Consider i.i.d.\ samples $X_1, X_2, \dots \sim \mu$ and $Y_1, Y_2, \dots \sim \nu$, arbitrarily corrupted to obtain $\tilde{X}_1, \tilde{X}_2, \dots$ and $\tilde{Y}_1, \tilde{Y}_2, \dots$. We measure the level of corruption by the rates
$\tau_\mu(n) = \frac{1}{n}\sum_{i=1}^n \mathds{1}_{\{X_i \neq \tilde{X_i}\}}$ and $\tau_\nu(n) = \frac{1}{n}\sum_{i=1}^n \mathds{1}_{\{Y_i \neq \tilde{Y_i}\}}$.
For this result, $\cX$ may be unbounded, and only the clean distributions $\mu$ and $\nu$ need compact support.
Let $\tilde{\mu}_n$ and $\tilde{\nu}_n$ denote the empirical measures associated with $\tilde{X}_1, \dots, \tilde{X}_n$ and $\tilde{Y}_1, \dots, \tilde{Y}_n$, respectively.

\begin{restatable}[Robust consistency]{proposition}{robustconsistency}
\label{prop:robust-consistency}
Fix $p < \infty$ and suppose that $\tau_\mu(n) \lor \tau_\nu(n) = O(n^{-a})$ almost surely, for some $a > 0$. Then, setting $\eps_n = n^{-b}$, for any $0 < b < a$, we have
\begin{equation*}
    \left|\Wp(\mu,\nu) - \Wp^{\eps_n}\left(\tilde{\mu}_n, \tilde{\nu}_n\right)\right| \to 0
\end{equation*}
as $n \to \infty$, almost surely.
\end{restatable}

That is, so long as the fraction of (potentially unbounded) corrupted samples converges to 0, there exists a schedule for robustness radii so that the correct distance is recovered. The proof in \cref{prf:robust-consistency} uses ideas from the previous result to alleviate potential problems arising from large corruptions.%

\begin{remark}[Median-of-means]
This consistency result mirrors that presented by \cite{staerman21} for a median-of-means (MoM) estimator $\mathcal{W}_\mathrm{MoM}$. They produce a robust estimate for $\Wone$ by partitioning the contaminated samples into blocks and replacing the each mean appearing in the $\Wone$ dual with a median of block means, where the number of blocks depends on the contamination fractions. We remark that when $\tau_\mu \lor \tau_\nu = \Omega(1)$ and contaminations are stochastic---i.e., each $\tilde{X}_i$ and $\tilde{Y}_i$ are sampled from some $\tilde{\mu} \in \cB_{\eps}(\mu)$ and $\tilde{\nu} \in \cB_{\eps'}(\nu)$, respectively---we have $\mathcal{W}_\mathrm{MoM}(\tilde{\mu}_n,\tilde{\nu}_n) \to \Wone(\tilde{\mu},\tilde{\nu})$. Hence, this approach cannot provide guarantees in the vein of \cref{sec:robust}, since it may be that $\Wone(\tilde{\mu},\tilde{\nu}) \gg \Wone(\mu,\nu)$.
\end{remark}

\section{Applications}\label{SEC:applications}

We now move to applications of the proposed robust OT framework. We first discuss computational aspects and provide an algorithm to compute $\RWp$ based on its dual form. The algorithm, which requires only a minor modification to standard OT solvers, is then used for generative modeling from contaminated data. %

\subsection{Computation}
\label{subsec:comp}
In practice, similarity between datasets is often measured using the so-called neural network (NN) distance $\mathsf{d}_\cF(\mu,\nu) = \sup_{\theta \in \Theta} \int f_\theta \dd \mu - \int f_\theta \dd \nu$, where $\cF=\{ f_\theta \}_{\theta \in \Theta}$ is a NN class \citep{arora2017}. Given two batches of samples $X_1, \dots, X_n \sim \mu$ and $Y_1, \dots, Y_n \sim \nu$, we approximate integrals by sample means and estimate the supremum via stochastic gradient ascent. Namely, we follow the update rule
\[\theta_{t+1} \gets \theta_t + \frac{\alpha_t}{n} \sum_{i=1}^n \big[ \nabla_\theta f_\theta(X_i) - \nabla_\theta f_\theta(Y_i) \big].\]
When $\cF$ approximates the class of 1-Lipschitz functions, we approach a Kantorovich dual and obtain an estimate for $\Wone(\mu,\nu)$, which is the core idea behind the WGAN \citep{arjovsky_wgan_2017} (see \cite{makkuva2020} for an extension to $\Wtwo$).

By virtue of our duality theory, OT solvers as described above can be easily adapted to the robust framework. For the purpose of generative modeling, the one-sided robust distance defined by
\begin{equation}
\label{eq:one-sided}
    \mathsf{W}_p^\eps(\mu \| \nu)\coloneqq\inf_{\substack{0 \leq \mu' \leq \mu\\ \mu'(\cX) = 1-\eps}}\Wp\left(\frac{\mu'}{1-\eps},\nu\right)
\end{equation} 
is most appropriate, since data may be contaminated but generated samples from the model are not. In \cref{app:asymmetric-results}, we translate our duality result to this setting,
finding that
\begin{align}
\label{eq:one-sided-dual}
\begin{split}
(1-\eps)\RWp(\mu \| \nu)^p = \sup_{f \in C_b(\cX)}\int\mspace{-3mu} f \dd \mu + (1- \eps)\mspace{-3mu}\int\mspace{-3mu} f^c \dd \nu - \eps \sup_{x \in \cX}f(x).
\end{split}
\end{align}
This representation motivates a modified gradient update for the corresponding NN distance\footnote{An inspection of the proof of \cref{thm:RWp-dual} reveals that a similar duality holds for NN distances, and, generally,~IPMs} estimate:
\begin{align*}
    \theta_{t+1} \gets \theta_t + \frac{\alpha_t}{n} \sum_{i=1}^n \big[ \nabla_\theta f_\theta(X_i) - (1 - \eps) \nabla_\theta f_\theta(Y_i) \big]
    - \eps \, \nabla_\theta f_\theta(X_{i^*}),\quad i^* \in \argmax_i f_\theta(X_i)
\end{align*}
(note here that we are formally computing Clarke subgradients \citep{clarke90}). For example, in a \texttt{PyTorch} implementation of WGAN, the modified update can be implemented with a one-line adjustment of code:

\newcommand\Small{\fontsize{13}{9.2}\selectfont}
\newcommand*\LSTfont{\Small\ttfamily\SetTracking{encoding=*}{-60}\lsstyle}
\begin{lstlisting}[language=Python,basicstyle=\LSTfont,deletendkeywords={max},columns=fullflexible]
score = f_data.mean() - f_generated.mean() # old
score = f_data.mean() - (1-eps)*f_generated.mean()  - eps*f_data.max() # new
\end{lstlisting}
Due to the non-convex and non-smooth nature of the objective, formal optimization guarantees seem challenging to obtain, and we defer this exploration for future work. Nevertheless, as we will see next, this approach proves quite fruitful in practice. Full experimental details for the following results and comparisons with existing work are provided in \cref{app:experiment-details}, and code is provided at \url{https://github.com/sbnietert/robust-OT}. Computations were performed on a cluster machine equipped with a NVIDIA Tesla V100. %

\subsection{Generative Modeling}
\label{subsec:gen_mod}

\begin{wrapfigure}{r}{0.5\textwidth}
\centering
\vspace{-5mm}
\hspace{2mm} Robust GANs \hspace{11mm} Standard GANs\\
\vspace{2mm}
\includegraphics[width=.48\linewidth]{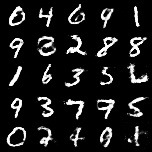}
\includegraphics[width=.48\linewidth]{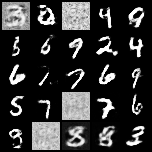}\\
\vspace{0.75mm}
\includegraphics[width=.48\linewidth]{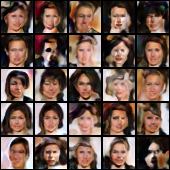}
\includegraphics[width=.48\linewidth]{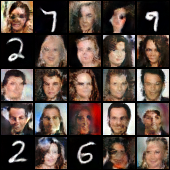}
\caption{
(top): samples generated by robustified (left) and standard (right) WGAN-GP after training on corrupted MNIST dataset. (bottom): samples generated by robustified (left) and standard (right) StyleGAN 2 after training on corrupted CelebA-HQ dataset (left).}\label{fig:generated-samples}
\end{wrapfigure}

We examine outlier-robust generative modeling using the modification suggested above. We train two WGAN with gradient penalty (WGAN-GP) %
models \citep{gulrajani2017improved} on a contaminated dataset with 80\% MNIST data and 20\% random noise, running both with a standard selection of hyper-parameters but adjusting one to compute gradient updates according to the robust objective with a selection of $\eps = 0.25$. In \cref{fig:generated-samples} (top), we display generated samples produced by both networks after processing 125k contaminated batches of 64 samples. The effect of outliers is clearly mitigated by training with the robustified objective.

One subtlety to the above is that WGAN-GP employs a regularizer $R(f_\theta)$ penalizing large gradients, rather than enforcing a Lipschitz constraint. In \cref{app:regularization}, we provide results demonstrating that our duality result still holds in the presence of many types of regularizers. This further motivates us to apply the robustness technique to more sophisticated GANs which incorporate additional regularization, in particular StyleGAN~2 \citep{karras2020}. We again train two off-the-shelf models using contaminated data---this time, 80\% CelebA-HQ face photos and 20\% MNIST data---with one tweaked to perform gradient updates according to the robust objective with $\eps = 0.25$. We present generated samples in \cref{fig:generated-samples} (bottom). Once again, the modified objective enables learning a model that is largely free of outliers despite being trained on a contaminated dataset.

\section{Summary and Concluding Remarks}\label{SEC:summary}

This paper introduced the outlier-robust Wasserstein distance $\RWp$, which measures proximity between probability distributions while discarding an $\eps$-fraction of outlier mass. We conducted a theoretical study of its structural and statistical properties, covering robustness guarantees, strong duality, characteristics of optimal perturbations and dual potentials, regularity in~$\eps$, and empirical convergence. The derived dual form amounts to a simple modification of classic Kantorovich duality that regularizes the objective w.r.t.\ the sup-norm of the potential function. This gave rise an elementary robustification technique for duality-based OT solvers (by introducing said penalty), which enables adapting computational methods for classic $\Wp$ to compute $\RWp$. Leveraging this, we demonstrated the utility of $\RWp$ for generative modeling with contaminated data.
 
Future research directions are abundant, both theoretical and practical. First, the derived duality can be leveraged for many high-dimensional real-world inference tasks where outlier-robustness is desired, although a large-scale empirical exploration is beyond the scope of this~work. 
We also aim to sharpen our statistical bounds and provide finite-sample robustness guarantees. In the one-sample case, we expect a tighter upper bound that depends on an upper (robust) intrinsic dimension to hold when only a small amount of high-dimensional mass is present. Indeed, high-dimensional regions are harder to sample and we therefore  expect $\RWp$ to treat those as outliers in the empirical approximation setting. 
The following generalizations of the considered robust framework are also of interest: (i) general transportation costs (many of our structural results immediately generalize); (ii) unbounded domains; and (iii) other base and constraining distances (e.g., IPMs, NN distances, etc.).

\subsubsection*{Acknowledgements}
The authors would like to thank Benjamin Grimmer for helpful conversations surrounding minimax optimization, as well as Jacob Steinhardt and Adam Sealfon for useful discussions on robust statistics. S. Nietert was supported by the National Science Foundation (NSF) Graduate Research Fellowship under Grant DGE-1650441. R. Cummings was supported in part by NSF grants CNS-1850187 and CNS-1942772 (CAREER), a Mozilla Research Grant, and a JPMorgan Chase Faculty Research Award. Z. Goldfeld was supported by NSF grants CCF-1947801 and CCF-2046018 (CAREER), and the 2020 IBM Academic Award.

\bibliographystyle{abbrvnat}
\bibliography{references}

\begin{appendices}
\crefalias{section}{appendix}
\crefalias{subsection}{appendix}

\section{Proofs of Main Results}
We begin with some further preliminaries and notation. When convenient, we write $\E_\mu[f(X)]$ for the expectation of $f(X)$ with $X \sim \mu$. Next, we recall the general definition of $\Wp$ between non-negative measures of equal (but potentially not unit) mass. Given two measures $\mu,\nu \in \cM_+(\cX)$ with $\mu(\cX) = \nu(\cX)$, let $\Pi(\mu,\nu) = \{ \mu(\cX) \pi : \pi \in \Pi(\mu/\mu(\cX),\nu/\nu(\cX)) \}$, and, for $p \in [1,\infty]$, define
\begin{equation*}
    \Wp(\mu,\nu) \coloneqq \inf_{\pi \in \Pi(\mu,\nu)} \|d\|_{L^p(\pi)}.
\end{equation*}
We further recall the two-potential version of Kantorovich duality, which states that for $p \in [1,\infty)$,
\begin{equation}
\label{eq:kantorovich-duality}
\Wp(\mu,\nu)^p = \sup_{\substack{f,g \in C_b(\cX) \\ f(x) + g(y) \leq d(x,y)^p}} \int f \dd \mu + \int g \dd \nu.
\end{equation}
Finally, we define the push-forward of a measure $\mu \in \cM_+(\cX)$ w.r.t.\ a measurable map $T:\cX \to \cY$ by $T_\# \mu (B) \coloneqq \mu(T^{-1}(B))$ for all measurable $B \subseteq \cY$. 

\subsection{Preliminary Results}

We begin with a useful fact and lemma.

\begin{fact}
\label{fact:ignore-shared-mass}
For $p \in [1,\infty]$, $\eps \in [0,1]$, and $\mu,\nu \in \cP(\cX)$, we have $\Wp(\mu,\nu) \leq \Wp(\mu - \mu \land \nu, \nu - \mu \land \nu)$.
\end{fact}
\begin{proof}
For any feasible coupling $\pi \in \Pi(\mu - \mu \land \nu, \nu - \mu \land \nu)$ for the second distance, we have that $\pi' = \pi + (\Id,\Id)_\#(\mu \land \nu) \in \Pi(\mu,\nu)$ is feasible for the first distance with $\|d\|_{L^p(\pi')} = \|d\|_{L^p(\pi)}$.
\end{proof}

The following is a helpful rewriting of the $\RWp$ primal problem.

\begin{lemma}
\label{lem:conservative-estimate}
For $p \in [1,\infty]$, $\eps \in [0,1]$, and $\mu,\nu \in \cP(\cX)$, we have
\begin{equation}
\RWp(\mu,\nu) = \inf_{\substack{\mu',\nu' \in \cP(\cX):\\ \mu \in \cB_\eps(\mu'), \: \nu \in \cB_\eps(\nu')}} \Wp(\mu',\nu') \label{eq:mass-removal-2}.
\end{equation}
\end{lemma}
\begin{proof}
Here, we use that optimal perturbed measures for the original primal \eqref{eq:RWp} may be taken to have mass exactly $1-\eps$ (since any feasible measures may be scaled down until this is the case, without changing the original objective due to its normalization). We further observe that $\mu \in \cB_\eps(\mu')$ if and only if $(1-\eps)\mu' \leq \mu$.
\end{proof}

\subsection{Proof of \cref{thm:population-limit-robustness}}
\label{prf:robustness}

Given $\cD \subseteq \cP(\cX)$ and $0 \leq \eps \leq 1$, define the moduli of continuity
\begin{align*}
    \mathfrak{m}_-(\cD,\eps) \coloneqq \sup_{\substack{\kappa,\kappa' \in \cD\\ \kappa \in \cB_\eps(\kappa')}} \Wp(\kappa,\kappa'), \qquad \mathfrak{m}_+(\cD,\eps) \coloneqq \sup_{\substack{\kappa \in \cD, \,\kappa' \in \cP(\cX) \\ \kappa \in \cB_\eps(\kappa')}} \Wp(\kappa,\kappa').
\end{align*}
The following lemma shows that these quantities characterize minimax risk, mirroring arguments from \cite{donoho88}.

\begin{lemma}
For any $\cD \subseteq \cP(\cX)$ and $\eps < 1$, we have
\begin{equation*}
    \frac{1}{2} \mathfrak{m}_-(\cD,\eps) \leq R(\cD,\eps) \leq 4\mathfrak{m}_+(\cD,2\eps).
\end{equation*}
\end{lemma}

\begin{proof}
For the upper bound, fix $\mu,\nu \in \cD$, and consider any corrupted versions $\tilde{\mu} \in \cB_\eps(\mu), \tilde{\nu} \in \cB_\eps(\mu)$. 
To estimate $\Wp$, take any $\hat{\mu},\hat{\nu}$ such that $\tilde{\mu} \in \cB_\eps(\hat{\mu})$ and $\tilde{\nu} \in \cB_\eps(\hat{\nu})$, and set $\hat{W}(\tilde{\mu},\tilde{\nu}) = \Wp(\hat{\mu},\hat{\nu})$.
Since $\|\hat{\mu} - \mu\|_\tv, \|\hat{\nu} - \nu\|_\tv \leq 4\eps$, the midpoint distributions $\mu' = \frac{1}{(\mu \land \hat{\mu})(\cX)} \mu \land \hat{\mu}$ and $\nu' = \frac{1}{(\nu \land \hat{\nu})(\cX)} \nu \land \hat{\nu}$ must satisfy $\mu,\hat{\mu} \in \cB_{2\eps}(\mu')$ and $\nu,\hat{\nu} \in \cB_{2\eps}(\nu')$. We thus bound
\begin{align*}
    |\Wp(\tilde{\mu},\tilde{\nu}) - \Wp(\mu,\nu)| &= |\Wp(\hat{\mu},\hat{\nu}) - \Wp(\mu,\nu)|\\
    &\leq \Wp(\hat{\mu},\mu') + \Wp(\mu',\mu) + \Wp(\hat{\nu},\nu') + \Wp(\nu',\nu)\\
    &\leq 4 \mathfrak{m}_+(\cD,2\eps).
\end{align*}
For the lower bound, fix any feasible $\kappa,\kappa' \in \cD$ with $\kappa \in \cB_\eps(\kappa')$, and suppose that both contaminated measures $\tilde{\mu}$ and $\tilde{\nu}$ are equal to $\kappa$. This is consistent with two cases: (1) the clean measures are $\mu = \nu = \kappa'$ (in which case $\Wp(\mu,\nu) = 0$); and (2) the clean measures are $\mu = \kappa$ and $\nu = \kappa'$ (in which case $\Wp(\mu,\nu) = \Wp(\kappa,\kappa')$). Hence, no matter which distance a robust proxy $\hat{W}$ assigns to $\tilde{\mu}$ and $\tilde{\nu}$, it must incur error at least $\Wp(\kappa,\kappa')/2$ in one of these cases, proving the lemma.
\end{proof}

Next, we bound these moduli for the family of interest. For the upper modulus, we borrow a standard result from the robust mean estimation literature.

\begin{lemma}
\label{lem:mean-resilience-bdd-moments}
Suppose $\cX = \R$ and $\mu \in \cD_q$ for $q \geq 1$. Then, for any $\nu \in \cP(\cX)$ such that $\mu \in \cB_\eps(\nu)$, we have $\left|\E_\mu[X] - \E_\nu[X]\right| \lesssim M \eps^{1 - 1/q} \land M (1-\eps)^{-1/q}$.
\end{lemma}
\begin{proof}
We apply Lemma E.2 of \cite{zhu2019resilience} with the Orlicz function $\psi(t) = t^{q}$, function class $\cF = \{ x \mapsto \pm x\}$, and $\eta = \eps$.
\end{proof}

\begin{lemma}
For any $q > p$ and $0 \leq \eps \leq 0.99$, we have
\begin{equation*}
    \mathfrak{m}_+(\cD_q,\eps) \lesssim M \eps^{1/p - 1/q}.
\end{equation*}
\end{lemma}

\begin{proof}
Without loss of generality, we suppose that $M=1$ (otherwise, one can always apply the lemma to the same metric space with distances --- and hence OT measurements --- shrunk by a factor of $M$). To control this modulus, fix any $\kappa \in \cD_q$ and $\kappa' \in \cP(\cX)$ such that $\kappa = (1-\eps)\kappa' + \eps \alpha$ for some $\alpha \in \cP(\cX)$. Take $y \in \cX$ with $\E_\kappa[d(X,\cdot)^q] \leq 1$. Writing $\tau = \eps \lor 1-\eps$, we then have
\begin{align*}
    \Wp(\kappa,\kappa') &\leq \eps^{1/p}\,\Wp(\kappa',\alpha)\\
    &\leq \eps^{1/p} \left( \Wp(\kappa',\delta_y) + \Wp(\alpha ,\delta_y)\right)\\
    &\leq 2 \eps^{1/p} \sup_{\substack{\beta \in \cP(\cX) \\ \kappa \in \cB_{\tau}(\beta)}} \E_\beta[d(X,y)^p]^{1/p}.
\end{align*}
The first inequality follows from \cref{fact:ignore-shared-mass} and the second follows from the triangle inequality for $\Wp$. To bound the remaining supremum, fix any $\beta$ such that $\kappa \in \cB_{\tau}(\beta)$. We then have
\begin{align*}
    \E_\beta[d(X,y)^p] &\leq  \E_\kappa[d(X,y)^p] + |\E_\beta[d(X,y)^p] -  \E_\kappa[d(X,y)^p]|\\
    &\leq 1 + |\E_\beta[d(X,y)^p] -  \E_\kappa[d(X,y)^p]|. 
\end{align*}
To bound $|\E_\beta[d(X,y)^p] -  \E_\kappa[d(X,y)^p]|$, we observe that for $X \sim \kappa$, the random variable $d(X,y)^p$ has bounded $q/p$th moments. Thus, \cref{lem:mean-resilience-bdd-moments} gives
\begin{equation*}
    |\E_\beta[d(X,y)^p] -  \E_\kappa[d(X,y)^p]| \lesssim \tau^{1 - p/q} \land (1-\tau)^{-p/q}.
\end{equation*}
Combining the previous bounds, we obtain
\begin{align*}
    \Wp(\kappa,\kappa') \lesssim 2\eps^{1/p}\bigl(1 + \tau^{1 - p/q} \land (1-\tau)^{-p/q}\bigr)^{1/p} \lesssim \eps^{1/p - 1/q},
\end{align*}
implying the lemma.
\end{proof}

\begin{lemma}
\label{lem:modulus-lower-bound}
For any $q > p$ and $0 \leq \eps < 1$,
\begin{equation*}
    \mathfrak{m}_-(\cD_q,\eps) \geq M \eps^{1/p - 1/q},
\end{equation*}    
so long as there exist $x,y \in \cX$ with $d(x,y) = M \eps^{-1/q}$.
\end{lemma}
\begin{proof}
Fix $\kappa = \delta_x$ and $\kappa' = (1-\eps)\delta_x + \eps \delta_y$. By design, both distributions belong to $\cD_q$ (since $\E_{\kappa'}[d(x,X)^q] = \eps d(x,y)^q = M$) and $\Wp(\kappa,\kappa') = \eps^{1/p}d(x,y) = M\eps^{1/p - 1/q}$.
\end{proof}

Finally, we prove the stated risk bound for $\RWp$.

\begin{lemma}
Fix $p \in [1,\infty]$ and $\eps \leq 1/4$. Take $\mu,\nu \in \cP(\cX)$ and let $\tilde{\mu} \in \cB_\eps(\mu), \tilde{\nu} \in \cB_\eps(\nu)$. Then, we have
\begin{equation}
\label{eq:RWp-risk-1}
    (1-3\eps)\Wp^{3\eps}(\mu,\nu) \leq \RWp(\tilde{\mu},\tilde{\nu}) \leq \Wp(\mu,\nu).
\end{equation}
Consequently, for $\mu,\nu \in \cD_q$ and $\tau = 1 - (1-3\eps)^{1/p} \in [3\eps/p,3\eps]$,
\begin{equation}
\label{eq:RWp-risk-2}
    |\RWp(\tilde{\mu},\tilde{\nu}) - \Wp(\mu,\nu)| \lesssim M \eps^{1/p - 1/q} + \tau \Wp(\mu,\nu)
\end{equation}
\end{lemma}
\begin{proof}
The second inequality of \eqref{eq:RWp-risk-1} follows directly from \cref{lem:conservative-estimate}. For the other direction, we begin with $\mu'_0,\nu'_0 \in \cM_+(\cX)$ feasible for the original primal \eqref{eq:RWp} for $\RWp(\tilde{\mu},\tilde{\nu})$, i.e., $\mu'_0 \leq \tilde{\mu}$, $\nu'_0 \leq \tilde{\nu}$ and $\mu'_0(\cX),\nu'_0(\cX) \geq 1-\eps$. Then, we intersect $\mu'_0$ with $(1-\eps)\mu$, intersect $\nu'_0$ with $(1-\eps)\nu$, and remove up to $\eps$ additional mass from each as needed to obtain $\mu'_1 \leq (1-\eps)\mu$ and $\nu'_1 \leq (1-\eps)\nu$ with equal mass such that $\Wp(\mu'_1,\nu'_1) \leq \Wp(\mu'_0,\nu'_0)$ and $\mu'_1(\cX) = \nu'_1(\cX) \geq 1-3\eps$. Dividing both measures by $1 - \eps$, we obtain $\mu'_2, \nu'_2$ feasible for $\Wp^{2\eps/(1-\eps)}(\mu,\nu)$ such that
\begin{align*}
    \Wp^{2\eps/(1-\eps)}(\mu,\nu) &\leq \Wp\!\left(\frac{\mu'_2}{\mu'_2(\cX)},\frac{\nu'_2}{\nu'_2(\cX)}\right)\\
    &\leq \left(\frac{1-\eps}{1-3\eps}\right)^{1/p} \, \Wp(\mu'_0,\nu'_0)\\
    &\leq \frac{(1-\eps)^{2/p}}{(1-3\eps)^{1/p}} \, \Wp\left(\frac{\mu'_0}{1-\eps},\frac{\nu'_0}{1-\eps}\right)\\
    &\leq (1-3\eps)^{-1/p} \, \Wp\left(\frac{\mu'_0}{1-\eps},\frac{\nu'_0}{1-\eps}\right)
\end{align*}
Infimizing over $\mu'_0$ and $\nu'_0$ gives $\Wp^{2\eps/(1-\eps)}(\mu,\nu) \leq (1-3\eps)^{-1/p}\,\RWp(\tilde{\mu},\tilde{\nu})$. Noting that $\Wp^{2\eps/(1-\eps)}(\mu,\nu) \geq \Wp^{3\eps}(\mu,\nu)$ gives the first inequality of \eqref{eq:RWp-risk-1}. To prove \eqref{eq:RWp-risk-2}, we first bound
\begin{align*}
   \Wp^{3\eps}(\mu,\nu) &= \inf_{\substack{0 \leq \mu' \leq \mu,\,0 \leq \nu' \leq \nu\\ \mu'(\cX) = \nu'(\cX) = 1-3\eps}} \Wp\!\left(\frac{\mu'}{1-3\eps},\frac{\nu'}{1-3\eps}\right)\\
   &\geq \Wp(\mu,\nu) - \sup_{\substack{0 \leq \mu' \leq \mu\\ \mu'(\cX) = 1-3\eps}} \Wp\left(\mu,\frac{\mu'}{1-3\eps}\right) - \sup_{\substack{0 \leq \nu' \leq \nu\\ \nu'(\cX) = 1 - 3\eps}} \Wp\left(\nu,\frac{\nu'}{1-3\eps}\right)\\
   &\geq \Wp(\mu,\nu) - 2\mathfrak{m}^+(\cD_q,3\eps).
\end{align*}
Finally, we compute
\begin{align*}
    \RWp(\tilde{\mu},\tilde{\nu}) &\geq (1-3\eps)\Wp^{3\eps}(\mu,\nu)\\
    &\geq (1 - \tau)\Wp(\mu,\nu) - 2\mathfrak{m}^+(\cD_q,3\eps)\\
    &\geq (1 - \tau)\Wp(\mu,\nu) - O(M\eps^{1/p - 1/q}),
\end{align*}
as desired.
\end{proof}

\subsection{Proof of \cref{cor:robustness-cov}}
\label{prf:robustness-cov}

For the upper bound, we observe that for $\kappa \in \cD_2^\mathrm{cov}$ with mean $\E_\kappa[X] = x_\kappa$, we have
\begin{align*}
    \E[\|x_\kappa - X\|_2^2] = \tr(\Sigma_\kappa) \leq d.
\end{align*}
Hence, we obtain the desired upper bound as an application of Theorem 1. For the other direction,
we apply \cref{lem:modulus-lower-bound} to lower bound $E(\cD_2^{\mathrm{cov}},\eps)$ by $\frac{1}{2}\Wp(\kappa,\kappa')$ where $\kappa = \delta_0$ and $\kappa' = (1-\eps)\delta_0 + \eps \cN(0,I/\eps)$. Indeed, both distributions belong to $\cD_2^{\mathrm{cov}}$ (since, in case (2), $\Sigma_{\nu} = I$) and $\Wp(\mu,\nu) \geq \eps^{1/p} \E_{X \sim \cN(0,I/\eps)}[\|X\|] = \Omega(\sqrt{d}\,\eps^{1/p - 1/2})$, as desired.

\subsection{Proof of \cref{prop:mass-addition}}
\label{prf:mass-addition}

\massaddition*

\begin{proof}
To begin, we rewrite the RHS as
\begin{equation}
\label{eq:mass-addition-step-1}
    \left(\frac{1+\eps}{1-\eps}\right)^{1/p} \inf_{\substack{\tilde{\mu} \in \cB_{\frac{\eps}{1+\eps}}(\mu) \\ \tilde{\nu} \in \cB_{\frac{\eps}{1+\eps}}(\nu)}} \Wp(\tilde{\mu},\tilde{\nu}) = \inf_{\substack{\mu' \geq \mu, \nu' \geq \nu \\ \mu'(\cX), \nu'(\cX) = 1 + \eps}} \Wp\left(\frac{\mu'}{1-\eps},\frac{\nu'}{1-\eps}\right) =: \overline{\mathsf{W}}_p^{\,\eps}(\mu,\nu).
\end{equation}
First, we prove that $\RWp(\mu,\nu) \geq \overline{\mathsf{W}}_p^{\,\eps}(\mu,\nu)$. Take $\mu' \leq \mu$ and $\nu' \leq \nu$ optimal for the original formulation \eqref{eq:RWp} of $\RWp(\mu,\nu)$, with $\mu'(\cX) = \nu'(\cX) = 1-\eps$ (see \cref{prop:minimizers} for existence of minimizers). Take $\pi \in \Pi(\mu',\nu')$ to be an optimal coupling for $\Wp(\mu',\nu')$. Then, consider the alternative perturbed measures $\mu'_+ \coloneqq \mu + (\nu - \nu')$ and $\nu'_+ \coloneqq \nu + (\mu - \mu')$, which are feasible for $\overline{\mathsf{W}}_p^{\,\eps}(\mu,\nu)$, and define the coupling $\pi_+ \in \Pi(\mu'_+,\nu'_+)$ by $\pi_+ = \pi + (\Id,\Id)_\#(\mu - \mu' + \nu - \nu')$. By construction, we have $ \|d\|_{L^p(\pi_+)} =  \|d\|_{L^p(\pi)}$, and so
\begin{equation*}
    \overline{\mathsf{W}}_p^{\,\eps}(\mu,\nu) \leq (1-\eps)^{-1/p} \|d\|_{L^p(\pi_+)} = (1-\eps)^{-1/p} \, \Wp(\mu',\nu') = \RWp(\mu,\nu),
\end{equation*}
as desired.

For the other direction, consider any $\mu' \geq \mu$ and $\nu' \geq \nu$ feasible for $\overline{\mathsf{W}}_p^{\,\eps}(\mu,\nu)$, and write $\mu' = \mu + \alpha, \nu' = \nu + \beta$ for $\alpha,\beta \in \cM_+(\cX)$ with $\alpha(\cX) = \beta(\cX) = \eps$. Take $\pi \in \cP(\mu',\nu')$ to be an optimal coupling for $\Wp(\mu',\nu')$, and let $\pi(y|x)$ be the regular conditional probability distribution such that $\nu'(\cdot) = \int_\cX \pi(\cdot|x) \dd \mu'(x)$. Informally, we next show that the added masses $\alpha$ and $\beta$ need not be moved during transport, since we might as well replace them with their destinations after transport. Formally, this requires a bit of labeling.

To start, we decompose $\nu'$ into $\nu' = \nu'_{\gets \mu} + \nu'_{\gets \alpha}$, where $\nu'_{\gets \mu}(\cdot) \coloneqq \int_\cX \pi(\cdot|x) \dd \mu(x)$ denotes the mass transported from $\mu$ to $\nu'$ via $\pi$ and $\nu'_{\gets \alpha}(\cdot) \coloneqq \int_\cX \pi(\cdot|x) \dd \alpha(x)$ denotes the mass transported from $\alpha$. Similarly, we decompose $\nu$ into $\nu_{\gets \mu} = \nu'_{\gets \mu} \land \nu$ and $\nu_{\gets \alpha} = \nu - \nu_{\gets \mu}$ and split $\beta$ into $\beta_{\gets \mu} = \nu'_{\gets \mu} - \nu_{\gets \mu}$ and $\beta_{\gets \alpha} = \beta - \beta_{\gets \mu}$. By this construction, we have
\begin{align*}
    \nu = \nu_{\gets \mu} + \nu_{\gets \alpha}, \quad \beta = \beta_{\gets \mu} + \beta_{\gets \alpha}, \quad \nu'_{\gets \mu} = \nu_{\gets \mu} + \beta_{\gets \mu}, \quad \nu'_{\gets \alpha} = \nu_{\gets \alpha} + \beta_{\gets \alpha}.
\end{align*}
Next, we arbitrarily decompose $\mu$ into $\mu_{\to \nu} + \mu_{\to \beta}$ and $\alpha$ into $\alpha_{\to \nu} + \alpha_{\to \beta}$ so that $\nu(\cdot) = \int_\cX \pi(\cdot|x) \dd(\mu_{\to \nu} + \alpha_{\to \nu})(x)$. To see that this is always possible, consider the restricted coupling $\bar{\pi} \in \Pi(\mu,\nu'_{\gets \mu})$ defined by $\bar{\pi}(A \times B) = \int_A \pi(B|x) \dd \mu(x)$, as well as the regular conditional probability distribution $\bar{\pi}(x|y)$ satisfying $\mu(\cdot) = \int_\cX \bar{\pi}(\cdot|y) \dd \nu'_{\gets \mu}(y)$. We can then set $\mu_{\to \nu}(\cdot) = \int_\cX \bar{\pi}(\cdot|y) \dd \nu_{\gets \mu}(y)$ and $\mu_{\to \beta} = \mu - \mu_{\to \nu}$. The same method works to construct $\alpha_{\to \nu}$ and $\alpha_{\to \beta}$.

 Now, consider the alternative perturbed measures $\tilde{\mu}' = \mu + \nu'_{\gets \alpha} = \mu_{\to \nu} + \nu'_{\gets \alpha} + \mu_{\to \beta}$ and $\tilde{\nu}' = \nu + \mu_{\to \beta} + \beta_{\gets \alpha} = \nu_{\gets \mu} + \nu'_{\gets \alpha} + \mu_{\to \beta}$ with  $\tilde{\pi} \in \Pi(\tilde{\mu}',\tilde{\nu}')$ defined by
\begin{align*}
    \tilde{\pi}(A \times B) &\coloneqq \int_A \pi(B|x) \dd \mu_{\to \nu}(x) + (\Id,\Id)_\#(\nu'_{\gets \alpha} + \mu_{\to \beta})(A \times B)\\
    &= \int_A \pi(B|x) \dd \mu_{\to \nu}(x) + (\mu_{\to \beta} + \nu'_{\gets \alpha})(A \cap B).
\end{align*}
Note that $\tilde{\mu}'$ and $\tilde{\nu}'$ are still feasible for the mass-addition problem and $\|d\|_{L^p(\tilde{\pi})} \leq \Wp(\mu',\nu')$, implying that $\|d\|_{L^p(\tilde{\pi})} = \Wp(\mu',\nu')$. Finally, observe that $\mu_{\to \nu} \leq \mu$ and $\nu_{\gets \mu} \leq \nu$ are feasible for the initial mass-subtraction problem \eqref{eq:RWp}, so
\begin{align*}
\RWp(\mu,\nu)^p &\leq \Wp\left(\frac{\mu_{\to \nu}}{\mu_{\to \nu}(\cX)}, \frac{\nu_{\gets \mu}}{\mu_{\to \nu}(\cX)} \right)^p\\
&= \frac{1}{\mu_{\to\nu}(\cX)} \, \Wp(\mu_{\to \nu}, \nu_{\gets \mu})^p\\
&\leq \frac{1}{\mu_{\to\nu}(\cX)} \|d\|_{L^p(\tilde{\pi})}^p\\
&= \frac{1}{\mu_{\to\nu}(\cX)} \Wp(\mu',\nu')^p\\
&=  \frac{1 - \eps}{\mu_{\to\nu}(\cX)} \overline{\mathsf{W}}_p^{\,\eps}(\mu,\nu)^p\\
&\leq \overline{\mathsf{W}}_p^{\,\eps}(\mu,\nu)^p,
\end{align*}
as desired.%
\end{proof}

\subsection{Proof of \cref{thm:RWp-dual}}
\label{prf:RWp-dual}

\RWpdual*

\begin{proof}
To start, we apply \cref{prop:mass-addition} and Kantorovich duality \eqref{eq:kantorovich-duality} to obtain
\begin{align}
    (1-\eps) \RWp(\mu,\nu)^p &= \inf_{\substack{\alpha,\beta \in \cM_+(\cX)\\ \alpha(\cX) = \beta(\cX) = \eps}} \Wp\left(\mu + \alpha,\nu + \beta \right)\\
    &= \inf_{\substack{\alpha,\beta \in \cM_+(\cX)\\ \alpha(\cX) = \beta(\cX) = \eps}} \sup_{\substack{f,g \in C_b(\cX) \\f(x)+g(y) \leq d(x,y)^p}}  \:  \int f \dd (\mu + \alpha) + \int g \dd (\nu + \beta).\label{eq:minimax}
\end{align}
By compactness of $\cX$, the infimum set is itself compact w.r.t.\ the weak topology. Having that, it is readily verified that the conditions for Sion's minimax theorem (convexity, continuity of objective, and compactness of infimum set) apply, giving
\begin{align*}
    (1-\eps) \RWp(\mu,\nu)^p &= \sup_{\substack{f,g \in C_b(\cX) \\f(x)+g(y) \leq d(x,y)^p}} \int f \dd \mu + \int g \dd \nu + \inf_{\alpha \in \cM_+(\cX), \alpha(\cX) = \eps} \int f \dd \alpha + \inf_{\beta \in \cM_+(\cX), \beta(\cX) = \eps} \int g \dd \beta\\
    &= \sup_{\substack{f,g \in C_b(\cX) \\f(x)+g(y) \leq d(x,y)^p}} \int f \dd \mu + \int g \dd \nu + \eps \left(\inf_{x \in \cX} f(x) + \inf_{y \in \cX} g(y) \right).
\end{align*}
Noting that replacing $g$ with $f^c$ preserves the constraints and can only increase the objective, we further have
\begin{align*}
      (1 - \eps) \RWp(\mu,\nu)^p &= \sup_{f \in C_b(\cX)}\: \int f \dd \mu + \int f^c \dd \nu - \eps\left(\sup_{x \in \cX} f(x) - \inf_{x \in \cX} f(x)\right)\\
       &= \sup_{f \in C_b(\cX)}\: \int f \dd \mu + \int f^c \dd \nu - \eps \Range(f),
\end{align*}
using the fact that $\inf_y f^c(y) = \inf_{x,y} d(x,y)^p - f(x) = - \sup_{x} f(x)$. The same reasoning allows us to restrict to $f = (f^c)^c$ if desired. Since adding a constant to $f$ decreases $f^c$ by the same constant, we are free to shift $f$ so that the final term equals $2\eps \|f\|_\infty$. Without shifting, we always have $\Range(f) \leq 2\|f\|_\infty$, so 
the problem simplifies to
\begin{align*}
    (1 - \eps) \RWp(\mu,\nu)^p &= \sup_{f \in C_b(\cX)}\: \int f \dd \mu + \int f^c \dd \nu - 2 \eps \|f\|_\infty,
\end{align*}
as desired. Since $\RWp(\mu,\nu) \leq \Wp(\mu,\nu)$, we can assume that the supremum set is uniformly bounded with $\|f\|_\infty \leq \Wp(\mu,\nu)/(2\eps)$. Furthermore, the argument preceding \cite[Proposition 1.11]{santambrogio2015} proves that the supremum set is uniformly equicontinuous. Since $\cX$ is compact, Arzel{\`a}–Ascoli implies that the supremum is achieved.

See \cref{app:asymmetric-results} for an extension of this result to the asymmetric setting.
\end{proof}

\subsection{Proof of \cref{prop:minimizers}}
\label{prf:minimizers-existence}

\minimizers*

\begin{proof}
We first prove existence via a compactness argument, and then turn to the lower envelope property $\mu',\nu' \geq \mu \land \nu$. This turns out to be considerably simpler to prove in the discrete setting, so we begin there and extend to the general case via a discretization argument. We already addressed the mass equality constraints in the proof of \cref{lem:conservative-estimate}. 

\paragraph{Existence:} 
Because measures in this feasible set are positive and bounded by $\mu$ and $\nu$, the set is tight, i.e. pre-compact w.r.t.\ the topology of weak convergence. If $p < \infty$, then for any sequence $\mu'_n,\nu'_n$ in the infimum set,
\begin{equation*}
    \lim_{R \to \infty} \limsup_{n \to \infty} \int_{d(x_0,x) \geq R} d(x_0,x)^p \dd \mu'_n(x) \leq \lim_{R \to \infty} \int_{d(x_0,x) \geq R} d(x_0,x)^p \dd \mu(x) = 0,
\end{equation*}
since $\mu \in \cP_p(\cX)$, with the same holding for the $\nu'_n$ sequence. Thus, by the characterization of convergence in $\Wp$ given by \cite[Theorem 6.8]{villani2009}, the infimum set is precompact w.r.t.\ $\Wp$. Finally, the mass function in the constraints is continuous w.r.t.\ the weak topology, so the infimum set is compact w.r.t.\ $\Wp$. As the objective is clearly continuous with w.r.t.\ $\Wp$, the infimum is achieved.

For $p=\infty$, we note that $\Winfty$ is lower semicontinuous on $\cP_\infty(\cX)$ w.r.t.\ the weak topology (and since supports are compact, the $\Wp$ topology, for all $p < \infty$). Indeed, for any $\mu'_n \stackrel{w}{\to} \mu'$, $\nu'_n \stackrel{w}{\to} \nu'$ in $\cP_\infty(\cX)$, we have
\begin{align*}
    \liminf_{n \to \infty} \Winfty(\mu'_n,\nu'_n) &=  \liminf_{n \to \infty} \sup_{p \to \infty} \Wp(\mu'_n,\nu'_n)\\
    &\geq \sup_{p \to \infty} \liminf_{n \to \infty} \Wp(\mu'_n,\nu'_n)\\
    &\geq \sup_{p \to \infty} \Wp(\mu',\nu') = \Winfty(\mu',\nu').
\end{align*}
The infimum of a lower semicontinuous function over a compact set is always achieved, as desired. %

\paragraph{Lower envelope (discrete case):} Having proven the existence of minimizers $\mu' \leq \mu$ and $\nu' \leq \nu$, it remains to show that we can take $\mu',\nu' \geq \mu \land \nu$. We begin with the case where $\cX$ is countable and treat measures and their mass functions interchangeably. Then, if $\mu' \geq \mu \land \nu$ fails to hold, there exists $x_0 \in \cX$ such that $\mu'(x_0) < \mu(x_0)$ and $\mu'(x_0) < \nu(x_0)$. We can further assume that $\nu(x_0) = \nu'(x_0)$. Otherwise, $a \coloneqq (\nu - \nu')(x_0) \land (\mu - \mu')(x_0)$ mass could be returned to both $\mu'$ and $\nu'$ at $x_0$ without increasing their transport cost. Indeed, for any $\pi \in \Pi(\mu',\nu')$, the modified plan $\pi + a \cdot \delta_{(x_0,x_0)}$ is valid for the new measures, $\mu'+a\delta_{x_0}$ and $\nu'+a\delta_{x_0}$ (still feasible by the choice of $a$), and attains the same cost.

Taking $\pi \in \Pi(\mu',\nu')$ optimal for $\Wp(\mu',\nu')$, we define the measure $\kappa \leq \mu'$ by $\kappa(x_0) \coloneqq 0$, $\kappa(x) \coloneqq \pi(x,x_0)$ for $x \neq x_0$. This captures the distribution of the mass that is transported away from $x_0$ w.r.t. $\pi$ when transporting $\nu'$ to $\mu'$.
By conservation of mass, $\kappa$ has total mass at least $(\nu' - \mu')(x_0) > 0$, so we can set $\tilde{\kappa} \leq \kappa$ to be a scaled-down copy of $\kappa$ with total mass $(\mu \land \nu' - \mu')(x_0) > 0$. Now, we define an alternative perturbed measure $\tilde{\mu}' \coloneqq \mu' + \tilde{\kappa}(\cX) \cdot \delta_{x_0} - \tilde{\kappa}$. By definition, we have $0 \leq \tilde{\mu}' \leq \mu'$ with $\tilde{\mu}'(\cX) = \mu'(\cX)$ and $\tilde{\mu}'(x_0) = (\mu \land \nu')(x_0)$. Furthermore, the modified plan $\tilde{\pi}$ defined by
\begin{align*}
    \tilde{\pi}(x,y) = \begin{cases}
        \pi(x,y), & y \neq x_0\\
        \frac{\kappa(\cX) - \tilde{\kappa}(\cX)}{\kappa(\cX)} \pi(x,x_0) , & x \neq x_0, y = x_0\\
        \pi(x_0,x_0) + \tilde{\kappa}(\cX), & x = y = x_0
    \end{cases}
\end{align*}
satisfies $\tilde{\pi} \in \Pi(\tilde{\mu}',\nu')$ and $\|d\|_{L^p(\tilde{\pi})} < \|d\|_{L^p(\pi)} = \Wp(\mu',\nu')$, contradicting the optimality of $\mu'$.

\paragraph{Lower envelope (general case):} For general $\cX$, we will need the following lemma, which allows us to apply our discrete argument to all settings.

\begin{lemma}[Dense approximation]
\label{lem:dense-approximation}
Let $(\cY,\rho)$ be a separable metric space with dense subset $D \subseteq \cY$. For any $y \in \cY$, let $y^\lambda$ denote a representative from $D$ with $d(y,y^\lambda) \leq \lambda$ (which exists by separability). Similarly, for any $K \subseteq \cY$, define $K^\lambda = \{ y^\lambda : y \in K\}$. Then, if $f:\cY \to \R$ is uniformly continuous, we have $\inf_{y \in K} f(y) = \lim_{\lambda \to 0} \inf_{y \in K^\lambda} f(y)$.
\end{lemma}
\begin{proof}
First, if $\{y_n\}_{n \in \N}$ is an infimizing sequence for $\inf_{y \in K} f(y)$, then
\begin{equation*}
    \liminf_{\lambda \to 0} \inf_{y \in K^\lambda} f(y) \leq \liminf_{n \to \infty} f(y_n^{1/n}) = \lim_{n \to \infty} f(y_n) = \inf_{y \in K} f(y),
\end{equation*}
where the second equality relies on the uniform continuity of $f$. Similarly, if $\{y_n\}_{n \in \N}$ with $y_n \in K^{\lambda_n}$, $\lambda_n \searrow 0$, is an limiting sequence for $\limsup_{\lambda \to 0} \inf_{y \in K^\lambda} f(x)$, then we can write $y_n = x_n^{1/n}$ for $x_n \in K$ and find
\begin{equation*}
    \limsup_{\lambda \to 0} \inf_{y \in K^\lambda} f(y) = \limsup_{n \to \infty} f(y_n) = \limsup_{n \to \infty} f(x_n) \geq \inf_{y \in K} f(y),
\end{equation*}
where the second equality again follows from the uniform continuity of $f$. The two inequalities imply the lemma.
\end{proof}

We will essentially apply this lemma to the space of measures $\cY = ((1-\eps)\cP_p(\cX), \Wp)$, letting $D$ be its dense subset of discrete measures. 
For any $\lambda > 0$, separability of $\cX$ implies the existence of a countable partition $\{A_i^\lambda\}_{i \in \N}$ of $\cX$ such that $\diam(A_i^\lambda) \leq \lambda$ for all $i \in \N$, with representatives $x_i^\lambda \in A_i^\lambda$ for all $i \in \N$. For any measure $\kappa \in \cY$, we let $\kappa^\lambda \in \cM_+(\cup_{i \in \N} \{x_i^\lambda\})$ denote the discretized measure defined by $\kappa^\lambda \coloneqq \sum_{i \in \N} \kappa(A_i^\lambda) \cdot \delta_{x_i^\lambda}$. For $p < \infty$, we have
\begin{equation*}
    \Wp(\kappa,\kappa^\lambda)^p \leq \sum_{i \in \N} \Wp(\kappa|_{A_i^\lambda},\kappa(A_i^\lambda) \cdot \delta_{x_i^\lambda})^p \leq \sum_{i \in \N} \kappa(A_i^\lambda) \lambda^p = \lambda^p,
\end{equation*}
and so $\Wp(\kappa,\kappa^\lambda) \leq \lambda$ (including $p = \infty$ by continuity). This discretization lifts to sets as in the lemma.
Now, for any $\alpha, \beta \in \cM_+(\cX)$ with $\alpha \leq \beta$, let $K_{\alpha,\beta} = \{ \kappa \in \cY : \alpha \leq \kappa \leq \beta, \kappa(\cX) = 1-\eps \}$. Importantly, this choice of discretization satisfies $K_{\alpha,\beta}^\lambda = K_{\alpha^\lambda,\beta^\lambda}$. Indeed if $\kappa \in K_{\alpha,\beta}$, then $\kappa^\lambda \in K_{\alpha^\lambda,\beta^\lambda}$ by our discretization definition, and total mass is preserved. Likewise, if $\kappa \in K_{\alpha^\lambda,\beta^\lambda}$, then we can consider
\begin{align*}
    \tilde{\kappa} \coloneqq \sum_{i : x_i \in \supp(\kappa')} \alpha|_{A_i^\lambda} + c_i (\beta - \alpha)|_{A_i^\lambda},
\end{align*}
where $c_i \in [0,1]$ are chosen such that $\tilde{\kappa}(\{i\}) = \kappa(A_i)$ for all $i$. By this construction, we have $\kappa = \tilde{\kappa}^\lambda$ and $\tilde{\kappa} \in K_{\alpha,\beta}$ with $\tilde{\kappa}(\cX) = \kappa(\cX) = 1-\eps$, as desired. We also observe that $\alpha^\lambda \land \beta^\lambda = (\alpha \land \beta)^\lambda$. Putting everything together, we finally obtain
\begin{align}
    (1-\eps)^{1/p} \,\RWp(\mu,\nu) &= \inf_{\substack{\mu' \in K_{0,\mu}\\\nu' \in K_{0,\nu}}} \Wp(\mu',\nu')\\
    &= \lim_{\lambda \to 0} \inf_{\substack{\mu' \in K_{0,\mu}^\lambda \\\nu' \in K_{0,\nu}^\lambda}} \Wp(\mu',\nu') \label{eq:discretization-step-2}\\
    &= \lim_{\lambda \to 0} \inf_{\substack{\mu' \in K_{0,\mu^\lambda} \\\nu' \in K_{0,\nu^\lambda}}} \Wp(\mu',\nu') \label{eq:discretization-step-3}\\
    &= \lim_{\lambda \to 0} \inf_{\substack{\mu' \in K_{\mu^\lambda \land \nu^\lambda,\mu^\lambda} \\\nu' \in K_{\mu^\lambda \land \nu^\lambda,\nu^\lambda}}} \Wp(\mu',\nu') \label{eq:discretization-step-4}\\
    &= \lim_{\lambda \to 0} \inf_{\substack{\mu' \in K_{(\mu \land \nu)^\lambda,\mu^\lambda} \\\nu' \in K_{(\mu \land \nu)^\lambda,\nu^\lambda}}} \Wp(\mu',\nu') \label{eq:discretization-step-5}\\
    &= \lim_{\lambda \to 0} \inf_{\substack{\mu' \in K^\lambda _{\mu \land \nu,\mu} \\\nu' \in K^\lambda_{\mu \land \nu,\nu}}} \Wp(\mu',\nu')\\
    &= \inf_{\substack{\mu' \in K_{\mu \land \nu,\mu} \\\nu' \in K_{\mu \land \nu,\nu}}} \Wp(\mu',\nu').
\end{align}
concluding the proof. Here, \eqref{eq:discretization-step-2} is an application of \cref{lem:dense-approximation} (technically, we use the product space $\cY \times \cY$ with metric $(\alpha,\beta),(\alpha',\beta') \mapsto \max\{\Wp(\alpha,\alpha'),\Wp(\beta,\beta')\}$), \eqref{eq:discretization-step-3} uses that $K_{\alpha,\beta}^\lambda = K_{\alpha^\lambda,\beta^\lambda}$, \eqref{eq:discretization-step-4} is an application of the discrete result, and the remaining steps apply the same results in reverse order. The final equality shows that the lower envelope of $\mu \land \nu$ can be assumed even in this general setting. %
\end{proof}

\subsection{Proof of \cref{prop:RWp-eps-dependence}}
\label{prf:RWp-eps-dependence}
\RWpepsdependence*

\begin{proof}
For (i), we clearly have $\Wp^{0} = \Wp$, and
\begin{equation*}
    \Wp^{\|\mu - \nu\|_\tv/2}(\mu,\nu) \leq (1-\|\mu - \nu\|_\tv/2)^{-1/p}\,  \Wp(\mu \land \nu, \mu \land \nu) = 0.
\end{equation*}

For (ii), let $\mu',\nu'$ be optimal for $\Wp^{\eps_1}(\mu,\nu)$ (in the mass removal formulation) and take $\pi \in \Pi(\mu',\nu')$ to be an optimal coupling for $\Wp(\mu',\nu')$. We can transform $\pi$ into a feasible coupling for $\Wp^{\eps_2}$ by restricting ourselves to the fraction $\frac{1 - \eps_1}{1-\eps_2}$ of mass $\pi'$ minimizing $\|d\|_{L^p(\pi')}$. This gives that 
\begin{align*}
    (1-\eps_2)^{1/p}\, \Wp^{\eps_2}(\mu,\nu) \leq \left(\frac{1-\eps_2}{1-\eps_1}\right)^{1/p} (1-\eps_1)^{1/p}\, \Wp^{\eps_1}(\mu,\nu) \iff \Wp^{\eps_2}(\mu,\nu) \leq \Wp^{\eps_1}(\mu,\nu),
\end{align*}
as desired.

For (iii), we recall from \citet[Theorem 6.13]{villani2009} that for $\alpha,\beta \in \cM_+(\cX)$ with $\alpha(\cX) = \beta(\cX)$, we have
\begin{equation*}
    \Wp(\alpha,\beta) \leq 2\diam(\cX) \|\alpha - \beta\|_\tv^{1/p}.
\end{equation*}
Since any pair of feasible measures for $\Wp^{\eps_2}$ are within (coordinate-wise) TV distance $\eps_2 - \eps_1$ from a pair of feasible measures for $\Wp^{\eps_1}$, this bound combined with the triangle inequality for $\Wp$ gives
\begin{align*}
    (1-\eps_2)^{1/p}\,\Wp^{\eps_2}(\mu,\nu) \geq  (1-\eps_1)^{1/p}\,\Wp^{\eps_1}(\mu,\nu) - 4\diam(\cX)(\eps_2 - \eps_1)^{1/p},
\end{align*}
which can be rearranged to give the desired inequality.
\end{proof}

\subsection{Proof of \cref{prop:RWp-coupling-primal}}
\label{prf:RWp-coupling-primal}

\RWpcouplingprimal*

\begin{proof}
Starting from formulation given by \cref{lem:conservative-estimate}, we compute
\begin{align*}
    \RWp(\mu,\nu) = \inf_{\substack{\mu' : \, \mu \in \cB_\eps(\mu') \\ \nu' :\, \nu \in \cB_\eps(\nu')}} \Wp(\mu',\nu') = \inf_{\substack{\mu' : \, \mu \in \cB_\eps(\mu') \\ \nu' :\, \nu \in \cB_\eps(\nu') \\ \pi \in \Pi(\mu',\nu')}} \|d\|_{L^p(\pi)} = \inf_{\substack{\pi \in \cP(\cX \times \cX) \\ \mu \in \cB_\eps(\pi_1) \\ \nu \in \cB_\eps(\pi_2)}} \|d\|_{L^p(\pi)},
\end{align*}
eliminating auxiliary variables $\mu'$ and $\nu'$ for the final equality.%
\end{proof}

\subsection{Proof of \cref{prop:coupling-delta-dependence}}
\label{prf:coupling-delta-dependence} 

\couplingdeltadependence*

\begin{proof}
In this case, because the two marginals of each $\pi_n$ are bounded by $\frac{\mu}{1-\eps}$ and $\frac{\nu}{1-\eps}$ respectively, this sequence is tight and admits a weakly convergent subsequence converging to some $\pi_\star \in \cP(\cX \times \cX)$.
Furthermore, $\pi_\star(A \times \cX) = \lim_{n \to \infty} \pi_1^n(A) = \mu(A)$, and, likewise, $\pi_\star(\cX \times B) = \nu(B)$; hence $\pi_\star \in \Pi(\mu,\nu)$.
We now recall the definition of $\Gamma$-convergence. %

\begin{definition}[$\Gamma$-Convergence]
Let $\cY$ be a metric space and consider a sequence $F_n : \cY \to \R \cup \{\infty\}$, $n \in \N$. We say that $\{F_n\}_{n \in \N}$ $\Gamma$-converges to $F:\cY \to \R \cup \{\infty\}$ and write $F_n \stackrel{\Gamma}{\to} F$~if
\begin{enumerate}[(i)]
\item For every $y_n, y \in \cY, n \in \N$, with $y_n \to y$, we have $F(y) \leq \liminf_{n \to \infty} F_n(y_n)$;
\item For any $y \in \cY$, there exists $y_n \in \cY, n \in \N$, with $y_n \to y$ and $F(y) \geq \limsup_{n \to \infty} F_n(y_n)$.
\end{enumerate}
\end{definition}

If $\{y_n\}_{n \in \N}$ is a sequence of minimizers for $F_n$, for each $n \in \N$, and $F_n \stackrel{\Gamma}{\to} F$, then any limit point of $\{y_n\}_{n \in \N}$ is a minimizer of $F$ \cite[Corollary 7.20]{dalmaso2012Gamma_conv}. Hence, it suffices to prove the $\Gamma$-convergence of $F_n:\cP(\cX \times \cX) \to \R \cup \{ \infty \}$ to $F:\cP(\cX \times \cX) \to \R \cup \{ \infty \}$ as defined by
\begin{align*}
    F_n(\gamma) &= \begin{cases}
    \|d\|_{L^p(\gamma)}, & \mu \in \cB_{\eps_n}(\gamma_1), \nu \in \cB_{\eps_n}(\gamma_2)\\
    \infty, & \text{o.w.}
    \end{cases}\\
    F(\gamma) &= \begin{cases}
    \|d\|_{L^p(\gamma)}, & \gamma \in \Pi(\mu,\nu)\\
    \infty, & \text{o.w.}
    \end{cases}
\end{align*}
For the $\liminf$ inequality, we start with any sequence $\gamma_n \in \cP(\cX \times \cX), n \in \N$, with weak limit $\gamma \in \cP(\cX \times \cX)$. If $\{\gamma_n\}_{n \in \N}$ does not contain a subsequence with $\mu \in \cB_{\eps_n}((\gamma_n)_1), \nu \in \cB_{\eps_n}((\gamma_n)_2)$, then the claim is trivial. Otherwise, $\gamma \in \Pi(\mu,\nu)$, and the Portmanteau Theorem gives
\begin{align*}
    F(\gamma) = \|d\|_{L^p(\gamma)} \leq \liminf_{n \to \infty} \| d \|_{L^p(\gamma_n)} = F_n(\gamma_n),
\end{align*}
where we used the fact that $d^p$ is non-negative and continuous for $p < \infty$. The fact that $L^p$ norms are continuous in $p \in [1,\infty]$ implies that the inequality holds for $p = \infty$ as well. %

For the $\limsup$ direction, fix $\gamma \in \Pi(\mu,\nu)$. Then, we must have $\mu \in \cB_{\eps_n}(\gamma_1), \nu \in \cB_{\eps_n}(\gamma_2)$ for all $n \in \N$, so we can consider the constant sequence $\gamma_n \equiv \gamma, n \in \N,$ and obtain
\begin{equation*}
    F_k(\gamma_n) = \|d\|_{L^p(\gamma)} = F(\gamma),
\end{equation*}
which implies the desired condition.%
\end{proof}

\subsection{Proof of \cref{prop:lp}}
\label{prf:LP}
\lp*

\begin{proof}
Assume first that $p < \infty$ and $\eps n = k$ for some $k \in \N$. Suppose that $\mu$ is uniform over $\{x_1, \dots, x_n\}$ and $\nu$ is uniform over $\{y_1, \dots, y_n\}$. Define the cost matrix $C \in \R^{n \times n}$ with $C_{ij} = d(x_i,y_j)^p$. Then, we can rewrite definition \eqref{eq:RWp} for $\RWp$ as a linear program, computing
\begin{align*}
    (1-\eps)\,\RWp(\mu,\nu)^p = \min_{\substack{\pi \in \R_+^{n \times n}\\ \pi \mathds{1} \leq \mathds{1}/n\\ \pi^\intercal \mathds{1} \leq \mathds{1}/n\\ \mathds{1}^\intercal \pi \mathds{1} \geq 1 - \eps}} \tr(\pi C) = \frac{1}{n} \min_{\substack{\pi \in \R_+^{n \times n}\\ \pi \mathds{1} \leq \mathds{1}\\ \pi^\intercal \mathds{1} \leq \mathds{1}\\ \mathds{1}^\intercal \pi \mathds{1} \geq n - k}} \tr(\pi C).
\end{align*}
This representation relates to the problem of characterizing the extreme points of the polytope of doubly substochastic $n \times n$ matrices with entries summing to an specified integer. This polytope was studied in \cite{cao2019}, where Theorem 4.1 shows that the extreme points of interest are exactly the partial permutation matrices of order $n - k$. This implies the existence of a coupling for $\RWp(\mu,\nu)^p$ whose marginals are uniform over $n - k = (1 - \eps)n$ points (giving mass $1/n$ to each point). 

When $\eps$ is not a multiple of $1/n$, the same result \cite[Theorem 4.1]{cao2019} reveals that there are optimal perturbed measures that each give mass $1/n$ to $\lfloor (1 - \eps)n \rfloor$ points and give the remaining mass to a single point. For $p$ sufficiently large, the set of minimizers to the above problem stabilizes to a constant set, so this argument captures the $p=\infty$ case as well.%
\end{proof}

\subsection{Proof of \cref{prop:RWp-dual-maximizers}}
\label{prf:RWp-dual-maximizers}

\RWpdualmaximizers*

\begin{proof}
Take $f \in C_b(\cX)$ maximizing \eqref{eq:RWp-dual} (by \cref{thm:RWp-dual}, such $f$ is guaranteed to exist). %
Take \emph{any} $\mu' = \mu - \alpha, \nu' = \nu - \beta$ optimal for the mass removal formulation of $\RWp(\mu,\nu)$, where $\alpha,\beta \in \cM_+(\cX)$ with $\alpha(\cX) = \beta(\cX) = \eps$. Of course, $\mu + \beta$ and $\nu + \alpha$ are then optimal for the mass addition formulation of $\RWp(\mu,\nu)$. Examining the proof of \cref{thm:RWp-dual}, we see that $\mu + \beta,\nu + \alpha$ and $f,f^c$ must be a minimax equilibrium for the problem given in \eqref{eq:minimax}. Consequently, we have
\begin{align}
    (1-\eps)\RWp(\mu,\nu)^p &= \int f \dd \mu + \int f^c \dd \nu + \int f \dd \beta + \int f^c \dd \alpha \label{eq:dual-maximizers-1}\\
    &=\int f \dd \mu + \int f^c \dd \nu + \eps \inf_x f(x) - \eps \sup_x f(x) \label{eq:dual-maximizers-2}.
\end{align}
Since $\inf_y f^c(y) = - \sup_{x} f(x)$ (and the minimizers of $f^c$ correspond to the maximizers of $f$),  we have that \eqref{eq:dual-maximizers-1} is strictly less than \eqref{eq:dual-maximizers-2} unless $\supp(\beta) \subseteq \argmin(f)$ and $\supp(\alpha) \subseteq \argmax(f)$. This suggests taking $\alpha = \mu|_{\argmax(f)}$ and $\beta = \nu|_{\argmin(f)}$, but we cannot do so in general; consider the case where $\mu$ and $\nu$ are uniform discrete measures on $n$ points and $\eps$ is not a multiple of $1/n$. Issues with that approach also arise when $\mu$ and $\nu$ are both supported on $\argmax(f)$ and the optimal $\mu'$ satisfies $\mu' \geq \mu \land \nu$.%
\end{proof}

\subsection{Proof of \cref{prop:loss-trimming}}
\label{prf:loss-trimming}

\losstrimming*

\begin{proof}
We place no restrictions on $\mu$ and $\nu$ initially and extend to the two-sided robust distance $\Wp^{\eps_1,\eps_2}$ defined in \cref{app:asymmetric-results}. For this result, we will apply Sion's minimax theorem to the mass-removal formulation of $\Wp^{\eps_1,\eps_2}$. Mirroring the proof of \cref{thm:RWp-dual}, we compute
\begin{align*}
    \Wp^{\eps_1,\eps_2}(\mu,\nu) &= \inf_{\substack{0 \leq \mu' \leq \mu\\ 0 \leq \nu' \leq \nu\\ \mu'(\cX) = 1-\eps_1 \\ \nu'(\cX) = 1-\eps_2}} \sup_{\substack{f,g \in C_b(\cX)\\ f(x) + g(y) \leq d(x,y)^p }} \frac{1}{1 - \eps_1} \int f \dd \mu' + \frac{1}{1 - \eps_2} \int g \dd \nu'\\
    &=  \sup_{\substack{f,g \in C_b(\cX)\\ f(x) + g(y) \leq d(x,y)^p }} \left( \frac{1}{1 - \eps_1} \inf_{\substack{0 \leq \mu' \leq \mu\\ \mu'(\cX) = 1-\eps_1}} \int f \dd \mu' + \frac{1}{1 - \eps_2} \inf_{\substack{0 \leq \nu' \leq \nu\\ \nu'(\cX) = 1-\eps_2}} \int g \dd \nu' \right)\\
    &=  \sup_{f \in C_b(\cX)} \left(\frac{1}{1 - \eps_1} \inf_{\substack{0 \leq \mu' \leq \mu\\ \mu'(\cX) = 1-\eps_1}} \int f \dd \mu' + \frac{1}{1 - \eps_2} \inf_{\substack{0 \leq \nu' \leq \nu\\ \nu'(\cX) = 1-\eps_2}} \int f^c \dd \nu'\right).
\end{align*}
When $\mu$ and $\nu$ are uniform distributions over $n$ points, and $\eps_1,\eps_2$ are multiples of $1/n$, we simplify further to
\begin{align*}
    \Wp^{\eps_1,\eps_2}(\mu,\nu) &=  \sup_{f \in C_b(\cX)} \left(\frac{1}{|\cA|} \min_{\substack{\cA \subseteq \supp(\mu)\\ |\cA| = (1-\eps_1)n}} \sum_{x \in A} f(x) + \frac{1}{|\cB|} \min_{\substack{\cB \subseteq \supp(\nu)\\ |\cB| = (1-\eps_2)n}} \sum_{y \in B} f^c(y) \right),
\end{align*}
as desired
\end{proof}

\subsection{Proof of \cref{prop:one-sample-rates}}
\label{prf:one-sample-rates}

\onesamplerates*

\begin{proof}
For the upper bound, we simply use that $\RWp(\mu,\nu) \leq \Wp(\mu,\nu)$ and apply well-known empirical convergence results for $\Wp$ which give $\E[\Wp(\mu,\hat{\mu}_n)] \leq C n^{-1/d}$ for a constant $C$ independent of $n$ (and $\eps$); see, e.g., \cite{weed2019}.

For the lower bound, if $d_*(\mu) > s$, then there is $\delta_0 > 0$ such that $\cN_{\delta}(\mu,\tau) \geq \delta^{-t}$ for all $\delta \leq \delta_0$. By \cref{prop:minimizers}, we know that $\RWp(\mu,\gamma)$ is achieved by $\mu' \leq \mu$ and $\gamma' \leq \gamma$, with $|\supp(\gamma')| \leq |\supp(\gamma)| \leq n$. Fix $\delta = n^{-s}/2$ and take $n$ sufficiently large so that $\delta \leq \delta_0$. Let $S = \cup_{x \in \supp(\gamma')} B(x,\delta/2)$, where $B(x,r)\coloneqq\{y\in\cX: d(x,y)\leq r\}$ is the ball of radius $r$ centered at $x$. Since $\cN(\mu,\tau) \geq \delta^{-t} > n$, we have that $\mu(S) \leq 1 - \tau$, which further implies $\mu'(S) \leq 1 - \tau + \eps$. 

For any $\pi \in \Pi(\mu',\gamma')$, we now have
\begin{equation*}
    \int d(x,y)^p \dd \pi(x,y) \geq \mu'(\cX \setminus S)\left(\frac{\delta}{2}\right)^p > (\tau - \delta)\left(\frac{\delta}{2}\right)^p = (\tau - \eps) \, 4^{-p}n^{-\frac{p}{s}},
\end{equation*}
which gives $\RWp(\mu,\gamma) = (1-\eps)^{-1/p}\, \Wp(\mu',\gamma') \geq \frac{1}{4}(\tau - \eps)^{1/p} (1 - \eps)^{-1/p}\, n^{-1/s}$. This suffices since $(1-\eps)^{-1/p} \geq 1$.
\end{proof}

\subsection{Proof of \cref{cor:simple-one-sample-rate}}
\label{prf:simple-one-sample-rate}

\simpleonesamplerate*

\begin{proof}
Set $\tau = (\int f \dd \lambda + \eps)/2$, so that $\eps < \tau < \inf f \dd \lambda$.
Then, for any covering of $\cX$ with balls of radius $\delta$ except for $A$ with $\mu(A) \leq \tau$, we have
\begin{align*}
    \int_A f \dd \lambda \leq \mu(A) \leq \tau \leq \int_\cX f \dd \lambda - c
\end{align*}
for a positive constant $c > 0$. This gives that $\int_{\cX \setminus A} f \dd \lambda > c$, and hence $\lambda(\cX \setminus A) \geq c/\|f\|_\infty > 0$. Since $\cX \setminus A$ has Lebesgue measure bounded away from 0, a volumetric argument implies that the covering must contain at least $C\delta^{-d}$ balls of radius $\delta$ for some $C > 0$. This implies the desired lower covering dimension and, by the proof of the previous lower bound, the desired empirical convergence rate.%
\end{proof}

\subsection{Proof of \cref{prop:two-sample-rate}}
\label{prf:two-sample-rate}

\twosamplerate*

\begin{proof}
To begin, we write
\begin{align*}
    (1-\eps)\RWp(\mu, \nu)^p &= \sup_{f \in C_b(\cX)} \mu(f) + \nu(f^c) - \Range(f) \\
    (1-\eps)\RWp(\hat{\mu}_n, \hat{\nu}_n)^p &= \sup_{f \in C_b(\cX)} -\hat{\mu}_n(f^c) - \hat{\nu}_n(f) - \Range(f),
\end{align*}
using that $C_b(\cX)$ is closed under negation. From here, we compute
\begin{align*}
   (1-\eps) \left| \RWp(\hat{\mu}_n, \hat{\nu}_n)^p - \RWp(\mu,\nu)^p \right| &\leq \sup_{f \in C_b(\cX)} \left| \int f \dd \mu + \int f^c \dd \hat{\mu}_n + \int f \dd \hat{\nu}_n + \int f^c \dd \nu \right|\\
    &\leq\Wp(\mu,\hat{\mu}_n)^p + \Wp(\nu,\hat{\nu}_n)^p.
\end{align*}
Hence, standard one-sample empirical convergence rates for $\Wp$ (e.g., \cite{weed2019}), give
\begin{equation*}
    \E\left[\left| \RWp(\hat{\mu}_n, \hat{\nu}_n)^p - \RWp(\mu,\nu)^p \right|\right] \leq C (1-\eps)^{-1} n^{-p/d} \leq 2 C n^{-p/d},
\end{equation*}
for a constant $C$ independent of $n$ and $\eps$.
\end{proof}

The standard lower bounds for two-sample empirical convergence \citep{talagrand92, barthe2013} do not immediately transfer to this setting. When $\mu = \nu$ are both an absolutely continuous distribution like $\Unif([0,1]^d)$, these can easily be adapted to provide a lower bound for
\begin{align*}
    \inf_{\substack{\cA \subseteq [n], |\cA| = (1-\eps)n \\ \cB \subseteq \supp([n]), |\cB| = (1-\eps)n}} &\E\left[\Wp\left(\frac{1}{(1-\eps)n}\sum_{i \in \cA} \delta_{X_i},\frac{1}{(1-\eps)n} \sum_{j \in \cB} \delta_{Y_j} \right)\right]\\
    = &\E\left[\Wp\left(\frac{1}{(1-\eps)n}\sum_{i=1}^{(1-\eps)n} \delta_{X_i},\frac{1}{(1-\eps)n} \sum_{j=1}^{(1-\eps)n} \delta_{Y_j} \right)\right].
\end{align*}
\makeatletter
\newcommand{\vast}{\bBigg@{4}}
\makeatother
For this to be helpful for the present case, one would also need to bound
\begin{multline*}
    \vast|\inf_{\substack{\cA \subseteq [n], |\cA| = (1-\eps)n \\ \cB \subseteq [n], |\cB| = (1-\eps)n}} \E\left[\Wp\left(\frac{1}{(1-\eps)n}\sum_{i \in \cA} \delta_{X_i},\frac{1}{(1-\eps)n} \sum_{j \in \cB} \delta_{Y_j} \right)\right]\\ - \E\left[\inf_{\substack{\cA \subseteq [n], |\cA| = (1-\eps)n \\ \cB \subseteq [n], |\cB| = (1-\eps)n}} \Wp\left(\frac{1}{(1-\eps)n}\sum_{i \in \cA} \delta_{X_i},\frac{1}{(1-\eps)n} \sum_{j \in \cB} \delta_{Y_j} \right)\right] \vast|,
\end{multline*}
which we defer for future work.

\subsection{Proof of \cref{prop:exact-recovery}}
\label{prf:exact-recovery}

\exactrecovery*

\begin{proof}
Under the conditions of the proposition, we will prove that $\mu' = (1-\eps)\mu$ and $\nu' = (1-\eps)\nu$ are the \emph{unique} minimizers for $\RWp(\tilde{\mu},\tilde{\nu})$, implying that $\RWp(\tilde{\mu},\tilde{\nu}) = \Wp(\mu,\nu)$. Suppose not, and take $\mu'$ and $\nu'$ optimal for \eqref{eq:RWp} such that $\mu' = \mu'_g + \mu'_b$, where $\mu'_g \coloneqq \mu' \land (1-\eps)\mu$ and $\mu'_b = \mu' - \mu'_g \leq \eps\alpha$, $\mu'_b \neq 0$. Let $\pi \in \Pi(\mu',\nu')$ be an optimal coupling for $\Wp(\mu',\nu')$, and let $\pi(y|x)$ be the regular conditional probability distribution such that $\nu'(\cdot) = \int_\cX \pi(\cdot|x) \dd \mu'(x)$. Then we can compute
\begingroup
\allowdisplaybreaks
\begin{align*}
    (1-\eps)\Wp(\mu,\nu)^p &\geq (1-\eps)\RWp(\tilde{\mu},\tilde{\nu})^p\\
    &= \Wp(\mu',\nu')^p\\
    &= \int\!\!\int d(x,y)^p \dd \pi(y|x) \dd \mu'(x)\\
    &= \int\!\!\int d(x,y)^p \dd \pi(y|x) \dd \mu'_g(x) + \int\!\!\int d(x,y)^p \dd \pi(y|x) \dd \mu'_b(x)\\
    &\geq (1-\eps)\Wp(\mu,\nu)^p - \int\!\!\int d(x,y)^p \dd \pi(y|x) \dd (\mu - \mu'_g)(x) + \int\!\!\int d(x,y)^p \dd \pi(y|x) \dd \mu'_b(x)\\
    &> (1-\eps)\Wp(\mu,\nu)^p - (\mu - \mu'_g)(\cX) \diam(S)^p + \mu'_b(\cX) \diam(S)^p\\
    &= (1-\eps)\Wp(\mu,\nu)^p,
\end{align*}
\endgroup
a contradiction. The same argument goes through if the three relevant distances are greater than $\|d\|_{L^\infty(\pi)}$, although general conditions under which this norm is finite seem hard to obtain (unless $p=\infty$, in which case $\|d\|_{L^\infty(\pi)} = \Winfty(\mu',\nu') \leq \Winfty(\mu,\nu)$).
\end{proof}

\subsection{Proof of \cref{prop:radius}}
\label{prf:radius}

\radius*

\begin{proof}
By \cref{prop:RWp-eps-dependence}, $(1-\delta)\Wp^{\delta}(\tilde{\mu},\tilde{\nu})^p$ is continuous and decreasing in $\delta$, so it must be differentiable except on a countable set. Now, if $0 \leq \delta' < \delta \leq \eps$, then the proof of \cref{prop:exact-recovery} reveals that there are optimal $\mu',\nu'$ for $\Wp^{\delta}(\tilde{\mu},\tilde{\nu})$ and $\mu'',\nu''$ for $\Wp^{\delta'}(\tilde{\mu},\tilde{\nu})$ such that $\tilde{\mu} \geq \mu'' \geq \mu' \geq (1-\eps)\mu$ and $\tilde{\nu} \geq \nu'' \geq \nu' \geq (1-\eps)\nu$. Letting $\pi \in \Pi(\mu'',\nu'')$ be an optimal coupling for $\Wp(\mu'',\nu'')$ and defining $\pi(y|x)$ as before, we compute
\begin{align*}
    (1-\delta)\Wp^{\delta}(\tilde{\mu},\tilde{\nu})^p - (1-\delta')\Wp^{\delta'}(\tilde{\mu},\tilde{\nu})^p &= \Wp(\mu',\nu')^p - \Wp(\mu'',\nu'')^p\\
    &= \Wp(\mu',\nu')^p - \|d\|_{L^p(\pi)}^p\\
    &\leq \int\!\!\int d(x,y) \pi(y|x) \dd \mu'(x) - \int\!\!\int d(x,y) \pi(y|x) \dd \mu''(x)\\
    &= - \int\!\!\int d(x,y) \pi(y|x) \dd (\mu'' - \mu')(x)\\
    &\leq - (\delta - \delta') D^p.
\end{align*}
Taking $\delta' \to \delta$, we find that the derivative of interest is bounded by $-D^p$ from above wherever it exists. On the other hand, if $\eps \leq \delta < \delta' \leq 1$, then there are optimal $\mu',\nu'$ for $\Wp^{\delta}(\tilde{\mu},\tilde{\nu})$ and $\mu'',\nu''$ for $\Wp^{\delta'}(\tilde{\mu},\tilde{\nu})$ such that $(1-\eps)\mu \geq \mu' \geq \mu'' \geq 0$ and $(1-\eps)\nu \geq \nu' \geq \nu''$. Letting $\pi \in \Pi(\mu'',\nu'')$ again be an optimal coupling for $\Wp(\mu'',\nu'')$, we compute
\begin{align*}
    (1-\delta')\Wp^{\delta'}(\tilde{\mu},\tilde{\nu})^p - (1-\delta)\Wp^{\delta}(\tilde{\mu},\tilde{\nu})^p &= \Wp(\mu'',\nu'')^p - \Wp(\mu',\nu')^p\\
    &= \|d\|_{L^p(\pi)}^p - \Wp(\mu',\nu')\\
    &\geq \|d\|_{L^p(\pi)}^p - \|d\|_{L^p(\pi)}^p - \Wp(\mu' - \mu'', \nu' - \nu'')^p\\
    &\geq - \diam(S)^p (\delta' - \delta).
\end{align*}
Taking $\delta' \to \delta$, we find that the derivative of interest is bounded by $-\diam(S)^p > -D^p$ from below, wherever it exists.
\end{proof}

\subsection{Proof of \cref{prop:robust-consistency}}
\label{prf:robust-consistency}

\robustconsistency*

\begin{proof}
Fix $p < \infty$ (a minor adjustment from the statement in the main text).
By design, there exists some $n_0 \in \N$ such that $\eps_n \geq \tau_\mu(n) \lor \tau_\nu(n)$ for all $n \geq n_0$. For such $n$,
let $\mu_n$ and $\nu_n$ denote the empirical probability measures associated with the uncontaminated samples among $\tilde{X}_1, \dots, \tilde{X}_n$ and $\tilde{Y}_1, \dots, \tilde{Y}_n$, respectively. By the proof of
\cref{prop:exact-recovery}, if $\mu'$ and $\nu'$ are optimal for $\Wp^{\eps_n}(\tilde{\mu}_n,\tilde{\nu}_n)$, and $\pi \in \Pi(\mu',\nu')$ is optimal for $\Wp(\mu',\nu')$, then $\|d\|_{L^\infty(\pi)} \leq \diam(\supp(\mu_n + \nu_n)) \leq \diam(\supp(\mu + \nu))$. Indeed, were any mass to be transported further than this by an optimal coupling, the source or destination would have to be contaminated, and the transport cost could be strictly improved by removing this contamination instead of inlier mass. %
Now, let $\mu'_g$ and $\nu'_g$ denote the restrictions of $\mu'$ and $\nu'$, respectively, to $\supp(\mu + \nu)$.
Hence, applying the same Wasserstein-TV bound used to prove \cref{prop:RWp-eps-dependence} (iii), we obtain
\begin{align}
    (1-\eps_n) \Wp(\mu_n,\nu_n)^p &\geq (1-\eps_n) \Wp^{\eps_n}(\tilde{\mu}_n,\tilde{\nu}_n)^p \label{eq:robust-consistency-prf-1}\\
    &= \Wp(\mu',\nu')^p\\
    &\geq \Wp(\mu'_g,\nu'_g)^p - \tau_\mu(n) \lor \tau_\nu(n)\, \|d\|_{L^\infty(\pi)}^p\label{eq:robust-consistency-prf-2}\\
    &\geq \Wp(\mu'_g,\nu'_g)^p - \eps_n \diam(\supp(\mu + \nu))\\
    &\geq (1-2\eps_n)\Wp^{2\eps_n}(\mu_n,\nu_n)^p - \eps_n \diam(\supp(\mu + \nu))\label{eq:robust-consistency-prf-3}\\
    &\geq (1-2\eps_n)\Wp(\mu_n,\nu_n)^p - C \eps_n \diam(\supp(\mu + \nu)),\label{eq:robust-consistency-prf-4}
\end{align}
for an absolute constant $C > 0$.
Here, \eqref{eq:robust-consistency-prf-1} follows by the under-estimation property of $\RWp$ under Huber contamination, \eqref{eq:robust-consistency-prf-1} holds since at most $\tau_\mu(n) \lor \tau_\nu(n) \leq \eps_n$ fraction of samples are removed from $\mu'$ and $\nu'$ to obtain their restrictions, \eqref{eq:robust-consistency-prf-3} observes that $\mu'_g$ and $\nu'_g$ are feasible for $\Wp^{2\eps_n}(\mu_n,\nu_n)$, and \eqref{eq:robust-consistency-prf-4} is an application of the bound mentioned above. %
Hence, $\lim_{n\to\infty} \Wp^{\eps_n}(\tilde{\mu}_n,\tilde{\nu}_n)  = \lim_{n \to \infty}\Wp(\mu_n,\nu_n) = \Wp(\mu,\nu)$ as $n \to \infty$, where the second equality follows from empirical convergence of $\Wp$.
\end{proof}

\section{Supplementary Results and Discussion}

\subsection{Asymmetric Variant}
\label{app:asymmetric-results}
First, we formally define the asymmetric robust $p$-Wasserstein distance with robustness radii $\eps_1$ and $\eps_2$ by
\begin{equation*}
    \Wp^{\eps_1,\eps_2}(\mu,\nu) = \ \inf_{\substack{\mu',\nu' \in \cM_+(\cX) \\ 0 \leq \mu' \leq \mu, \: \|\mu - \mu'\|_\tv \leq \eps_1 \\
    0 \leq \nu' \leq \nu, \: \|\nu - \nu'\|_\tv \leq \eps_2}} \Wp\left(\frac{\mu'}{\mu'(\cX)},\frac{\nu'}{\nu'(\cX)}\right),
\end{equation*}
naturally extending the definition in \eqref{eq:RWp} so that $\RWp = \Wp^{\eps,\eps}$ and $\Wp(\mu \| \nu) = \Wp^{\eps,0}(\mu,\nu)$.
Our robustness results translate immediately to this setting. Indeed, writing the asymmetric minimax error
\begin{equation*}
    E(\cD,\eps_1,\eps_2) \coloneqq \inf_{\hat{W}:\cP(\cX)^2 \to \R} \sup_{\substack{\mu,\nu \in \cD\\ \tilde{\mu} \in \cB_{\eps_1}(\mu)\\ \tilde{\nu} \in \cB_{\eps_2}(\nu)}} \big|\hat{W}(\tilde{\mu},\tilde{\nu}) - \Wp(\mu,\nu)\big|,
\end{equation*}
we trivially obtain $E(\cD,\eps_1,\eps_2) \leq E(\cD,\eps_1 \lor \eps_2)$ since $\cB_{\eps}(\mu) \subset \cB_{\eps'}(\mu)$ for $\eps' \geq \eps$. Moreover, it is easy to check that $E(\cD,\eps_1,\eps_2) \gtrsim \mathfrak{m}_-(\cD,\eps_1 \lor \eps_2)$, as defined in \cref{lem:modulus-lower-bound} (using the same proof), and so our lower bounds carry over.

We next focus on obtaining dual forms for $\Wp^{\eps_1,\eps_2}$ which are relevant for applications, although most of the primal properties shown for $\RWp$ translate to this case as well.
Our proof of \cref{prop:loss-trimming} is already written for this general setting. Next, although the proof of \cref{prop:mass-addition} was given for the symmetric distance,
an identical argument reveals that the following variant holds in general:
\begin{equation*}
    \RWp(\mu,\nu) = \left(\frac{1-\eps_1\eps_2}{(1-\eps_1)(1-\eps_2)}\right)^{1/p} \hspace{-2mm} \inf_{\substack{\mu' \in \cB_{\delta_1}(\mu)\\ \nu' \in \cB_{\delta_2}(\nu)}} \Wp(\mu',\nu'),
\end{equation*}
where $\delta_1 = \frac{\eps_2(1-\eps_1)}{1-\eps_1\eps_2}$ and $\delta_2 = \frac{(1-\eps_2)\eps_1}{1-\eps_1\eps_2}$.
We now translate \cref{thm:RWp-dual} to the asymmetric case, following the same argument of the original proof to obtain
\begin{equation*}
    (1-\eps_1)(1-\eps_2)\Wp^{\eps_1,\eps_2}(\mu,\nu)^p = \sup_{f \in C_b(\cX)} (1-\eps_2) \int f \dd \mu + (1-\eps_1) \int f^c \dd \nu + (1-\eps_1)\eps_2 \inf_x f(x) + (1-\eps_2)\eps_1 \sup_x f(x).
\end{equation*}
Note that this matches the symmetric case when $\eps_1 = \eps_2$ and matches the one-sided dual \eqref{eq:one-sided-dual} when $\eps_2 = 0$.

\subsection{Relation to TV-Robustified \texorpdfstring{$\Wp$}{Wp}}
\label{app:equivalent-formulations}
In \citet{balaji2020}, the authors consider
\begin{equation*}
    \Wp^{\eps,\tv}(\mu,\nu) \coloneqq \inf_{\substack{\mu',\nu' \in \cP(\cX) \\ \|\mu - \mu'\|_\tv, \|\nu - \nu'\|_\tv \leq \eps}} \Wp(\mu',\nu'),
\end{equation*}
(although $\chi^2$ robustification is examined more thoroughly). We have the following relationship between this quantity and our robust distance.

\begin{proposition}[Relation to $\Wp^{\eps,\tv}$]
For $p \in [1,\infty]$ and $\mu,\nu \in \cP_p(\cX)$,
\begin{equation*}
    \Wp^{2\eps,\tv}(\mu,\nu) \leq \RWp(\mu,\nu) \leq (1-\eps)^{-1/p} \,\Wp^{\eps,\tv}(\mu,\nu).
\end{equation*}
\end{proposition}
\begin{proof}
To prove $\Wp^{2\eps,\tv} \leq \RWp$, take $\mu',\nu' \in \cP(\cX)$ optimal for \eqref{eq:mass-removal-2}, with $\mu \in \cB_\eps(\mu')$ and $\nu \in \cB_\eps(\nu')$. Writing $\mu = (1-\eps)\mu' + \eps \alpha$, for some $\alpha \in \cP(\cX)$, we compute $\|\mu - \mu'\|_\tv = \eps\|\mu' - \alpha\|_\tv \leq 2\eps$; similarly, $\|\nu - \nu'\|_\tv \leq 2\eps$. Hence, $\Wp^{2\eps,\tv}(\mu,\nu) \leq \Wp(\mu',\nu') = \RWp(\mu,\nu)$.

For the opposite inequality, take any $\mu',\nu' \in \cP(\cX)$ with $\|\mu - \mu'\|_\tv, \|\nu - \nu'\|_\tv \leq \eps$. Then, we can write $\mu' = \mu + \frac{\eps}{2}(\alpha_+ - \alpha_-)$ and $\nu' = \nu + \frac{\eps}{2}(\beta_+ - \beta_-)$, where $\alpha_+,\alpha_-,\beta_+,\beta_- \in \cP(\cX)$. Next, set
\begin{equation*}
\tilde{\mu}' = \frac{1}{1 + \eps}\mu + \frac{\eps}{1 + \eps}\left(\frac{\alpha_+ + \beta_-}{2}\right) \in \cB_{\frac{\eps}{1+ \eps}}(\mu), \qquad \tilde{\nu}' = \frac{1}{1 + \eps}\nu + \frac{1}{1 + \eps}\left(\frac{\beta_+ + \alpha_-}{2}\right) \in \cB_{\frac{\eps}{1+ \eps}}(\nu),
\end{equation*}
and observe that
\begin{equation*}
    \Wp(\tilde{\mu}',\tilde{\nu}')^p = \frac{1}{1+\eps} \Wp\left(\mu + \frac{\eps}{2}(\alpha_+ + \beta_-), \nu + \frac{\eps}{2}(\beta_+ + \alpha_-)\right)^p \leq \frac{1}{1+\eps}\Wp(\mu',\nu'),
\end{equation*}
where the last inequality follows by considering the transport plan which leaves $\beta_-$ and $\alpha_-$ stationary while otherwise following an optimal plan for $\Wp(\mu',\nu')$. Finally, noting that $\tilde{\mu}'$ and $\tilde{\nu}'$ are feasible for the mass-addition formulation \eqref{eq:mass-addition}, we obtain
\begin{equation*}
    \Wp^{\eps}(\mu,\nu) \leq \left(\frac{1+\eps}{1-\eps}\right)^{1/p} \Wp(\tilde{\mu}',\tilde{\nu}') \leq (1-\eps)^{-1/p} \, \Wp(\mu',\nu').
\end{equation*}
Infimizing over all such $\mu'$ and $\nu'$ gives $\Wp^{\eps}(\mu,\nu) \leq (1-\eps)^{-1/p}\, \Wp^{\eps,\tv}(\mu,\nu)$.
\end{proof}

\subsection{Examples}
\label{app:examples}
On the real line, $\RWp$ is often simple enough to compute, although it does not lend itself to many closed-form solutions. One useful class of examples is that of $\mu$ and $\nu$ which are absolutely continuous with respect to Lebesgue and whose supports have disjoint convex hulls, e.g., when $\mu$ is supported on the negative numbers and $\nu$ is supported on the positive numbers. In this case, due to the structure of optimal couplings for real distributions, we know that the optimal $\mu'$ and $\nu'$ simply cut off the extreme ends of the original measures. Assuming that the support of $\mu$ is to the left of that of $\nu$, we find that $\mu' = \mu|_{[a,\infty)}$ and $\nu' = \nu|_{(-\infty,b]}$ for constants $a,b$ such that $\mu'(\cX) = \nu'(\cX) = 1 - \delta$. For example, if $\mu = \mathrm{Exp}(\lambda)$ and $\nu = (-\Id)_\# \mu$, then $\Winfty(\mu,\nu) = \infty$ but $\RWinfty(\mu,\nu) \leq 2 \ln(1/\eps)/\lambda$.

\subsection{Experiment Details and Comparison with Existing Work}
\label{app:experiment-details}
Full code is available on GitHub at \url{https://github.com/sbnietert/robust-OT}. We stress that we never had to adjust hyper-parameters or loss computations from their defaults, only adding the additional robustness penalty and procedures to corrupt the datasets. The implementation of WGAN-GP used for the robust GAN experiments in the main text was based on a standard PyTorch implementation \citep{cao2017}, as was that of StyleGAN 2 \citep{labml}. The images presented in \cref{fig:generated-samples} were generated without any manual filtering after a predetermined number of batches (125k, batch size 64, for WGAN-GP; 100k, batch size 32, for StyleGAN 2). Training for the WGAN-GP and StyleGAN 2 experiments took 5 hours and 20 hours of compute, respectively, on a cluster machine equipped with a NVIDIA Tesla V100 and 14 CPU cores.

In addition to the experiments provided in the main text, we also compared our robustification method for GANs with the existing techniques of \citet{balaji2020} and \citet{staerman21} (noting that the GAN described in \cite{mukherjee2021} does not appear to scale to high-dimensional image data). Using publicly available code for these projects, we trained these robust GAN variants with their provided default options (using the continuous weighting scheme, recommended robustness strength, and DCGAN architecture for that of \citet{balaji2020}) using the 50k CIFAR-10 training set of 32x32 images contaminated with 2632 images with uniform noise in each pixel (giving a contamination fraction of $0.05$). In the table and figure above, the control GAN is that of \citet{balaji2020} without its robustification weights enabled, and our robust GAN is obtained by adding our robustness penalty with $\eps=0.07$ (two lines of code, as described in \cref{SEC:applications}) to this control. Each GAN was trained for 100k epochs, taking approximately 12 hours of compute on a cluster machine equipped with a NVIDIA V100 GPU and 14 CPU cores. In addition to the raw generated images presented in \cref{fig:generated-samples-2}, we also computed Frechet inception distance (FID) scores \citep{heusel2017} to track progress during training, provided in \cref{table:fids}. Our method performs favorably without tuning, but more detailed empirical study is needed to separate the impact of various hyperparameters (for example, the poor relative performance of \cite{staerman21} may be due in part to its distinct architecture).

\begin{table}
\centering
\begin{tabular}{lcccc}
\toprule
& \multicolumn{4}{c}{Robust GAN Method}
\\\cmidrule(lr){2-5}
 Epoch          & Ours  & \cite{balaji2020} & \cite{staerman21} & Control \citep{balaji2020} \\\midrule
20k    & 38.61 & 60.92 & 58.19 & 59.19 \\
60k & 22.08 & 35.49 & 49.64 & 39.08\\
100k & 18.87 & 30.44 & 40.81 & 31.40\\\bottomrule
\end{tabular}
\caption{FID scores between uncontaminated CIFAR-10 image dataset and datasets generated by various robust GANs during training.}
\label{table:fids}
\end{table}

{
\floatsetup[figure]{capposition={bottom}}
\begin{figure}
{
\small
\hspace{2mm} Ours \hspace{15mm} \cite{balaji2020} \hspace{1mm} \cite{staerman21} \hspace{10mm} Control
}
\begin{center}
\includegraphics[width=.23\linewidth]{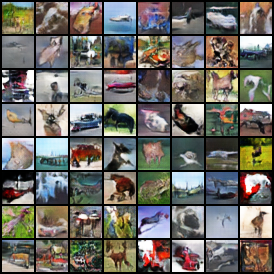}
\includegraphics[width=.23\linewidth]{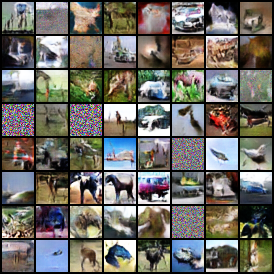}
 \includegraphics[width=.23\linewidth]{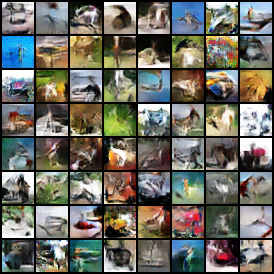}
\includegraphics[width=.23\linewidth]{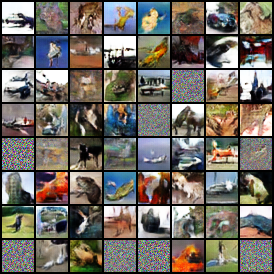}
\end{center}
\caption{Samples generated by various robust GANs after 100k epochs of training.}\label{fig:generated-samples-2}
\end{figure}
}

\subsection{Presence of Regularization}
\label{app:regularization}

Consider the Banach space $C_b(\cX)$ equipped with the uniform topology. Given a convex family of functions $\cF \subseteq C_b(\cX)$ and a convex, lower semicontinuous regularizer $R:\cF \to \R_+$, consider the statistical discrepancy measure between $\mu,\nu \in \cP(\cX)$ defined by
\begin{align*}
    d_{\cF,R}(\mu,\nu) = \sup_{f \in \cF} \int f \dd (\mu - \nu) - R(f).
\end{align*}
Then, we can verify that the conditions of Sion's minimax theorem apply as in \cref{thm:RWp-dual} to obtain
\begin{align*}
    \inf_{\substack{\mu' \geq \mu, \mu'(\cX) = 1 + \eps\\ \nu' \geq \nu, \nu'(\cX) = 1 + \eps}} d_{\cF,R}(\mu',\nu') &= \inf_{\substack{\alpha,\beta \in \cP(\cX)}} \sup_{f \in \cF} \int f \dd (\mu - \nu) + \eps \int f \dd (\alpha - \beta) + R(f)\\
    &= \sup_{f \in \cF} \int f \dd (\mu - \nu) + \eps \inf_{\alpha,\beta \in \cP(\cX)} \int f \dd (\alpha - \beta) + R(f)\\
    &= \sup_{f \in \cF} \int f \dd (\mu - \nu) - \eps \Range(f) - R(f)
\end{align*}
As an example, this applies when $\cF$ is a neural-net family and $R(f) = \|\nabla f\|_{L^p(\kappa)}^p$ for some $\kappa$. For the case of WGAN-GP, where $R(f) = \lambda \E_{x \sim \kappa}\left[(\|\nabla f(X)\|_2 - 1)^2\right]$, convexity may fail to hold, but the final expression above still serves as a non-trivial lower bound for the robustified discrepancy measure. In this particular setting, near maximizers $f \in \cF$ are approximately 1-Lipschitz for sufficiently large $\lambda$, suggesting that our original theory should apply up to some approximation error. We defer formal guarantees for non-convex regularizers for future work.

\section{Applications to Pufferfish Privacy}
\label{app:privacy}

Pufferfish privacy (PP) \citep{kifer2014,song2017,zhang2020} is a general privacy framework, which generalizes the popular notion of differential privacy \citep{DMNS06}. For a measurable data space $\cX$, the framework consists of three components: (i) a collection $\cS$ of measurable subsets of $\cX$ called secrets (e.g., facts about the database that one might wish to hide); (ii) a set of discriminative pairs $\mathcal{Q} \subseteq \cS \times \cS$ (i.e., pairs of secret events $(S,T)$ that should be indistinguishable); and (iii) a class of potential data distributions $\Theta \subseteq\cP(\cX)$.
For any $\theta\in\Theta$, let $X_\theta\sim\theta$ denote the dataset sampled from that distribution. A randomized algorithm satisfies PP w.r.t.\ $(\cS,\cQ,\Theta)$ if all pairs of secrets in $\cQ$ are indistinguishable as follows.

\begin{definition}[Pufferfish Privacy \citep{kifer2014}]\label{def:pufferfish}
A (possibly randomized) mechanism $\mathsf{M}:\cX\to\cY$ is $(\eps,\delta)$-Pufferfish private in a framework $(\cS,\cQ,\Theta)$ if for all data distributions $\theta\in \Theta$, secret pairs $(S,T)\in \mathcal{Q}$ with $\PP(X_\theta \in S) > 0$ and  $\PP(X_\theta \in T) > 0$, and measurable $Y \subseteq \cY$, we have
\begin{equation*}
    \PP\big(\mathsf{M}(X_\theta)\in Y\big|S\big)\le e^\eps \PP\big(\mathsf{M}(X_\theta)\in Y\big|T\big)+ \delta,
\end{equation*}
where $\PP$ is the underlying joint probability measure.
\end{definition}

\begin{remark}[Special cases]
A notable special case of PP is differential privacy \citep{DMNS06}, where $\cS$ is the set of all possible databases, $\cQ$ contains all pairs that differ in a single entry, and $\Theta = \cP(\cX)$ (i.e., no distributional assumptions are made, and privacy is guaranteed in the worst case). Another special case is \emph{attribute privacy} \citep{zhang2020}, which protects global statistical properties of data attributes. In this case, $\cS$ is the value of a function evaluated on the data, $\mathcal{Q}$ contains pairs of function values, and $\Theta$ captures assumptions on how the data were sampled and correlations across attributes in the data.
\end{remark}

A general approach for publishing the value of a function $f:\cX\to\RR$ on a dataset $X_\theta$ with $(\eps,0)$-PP is the \emph{Wasserstein Mechanism} \citep{song2017}. %
This mechanism computes the true value $f(X_\theta)$, and then adds Laplace noise that scales with the maximal $\Winfty$ distance between conditional distributions of $f(X_\theta)$ given any pair of secrets from $\mathcal{Q}$. Denoting this maximized distance by $W$, the mechanism outputs $\mathsf{M}(X_\theta)=f(X_\theta)+Z$, where $Z\sim\mathrm{Lap}(W/\eps)$, which can be shown to be $(\eps,0)$-PP for any framework $(\cS,\cQ,\Theta)$; see Algorithm \ref{alg.wass} with $\delta=0$. %
The privacy guarantee relies on the fact that $\Winfty(\mu,\nu) \leq \eps$ if and only if there exists a coupling $(A,B)$ of $\mu$ and $\nu$ such that $|A-B| \leq \eps$. If we are instead interested in $(\eps,\delta)$-PP, we only need $|A-B| \leq \eps$ to hold with probability $1-\delta$. We next show that this property is characterized precisely by $\Winfty^\delta$. %

\begin{lemma}[Coupling property]
\label{lem:coupling}
If $\Winfty^\delta(\mu,\nu) \leq \eps$, then there exists a coupling $(X,Y)$ for $\mu,\nu$ such that $|X-Y| \leq \eps$ with probability at least $1 - 2\delta$. Conversely, if there exists such a coupling, then $\Winfty^{2\delta}(\mu,\nu) \leq \eps$.
\end{lemma}

\begin{proof}
If $\Winfty^\delta(\mu,\nu) \leq \eps$, then there exists $\nu',\nu' \in \cP(\cX)$ such that $\mu \in \cB_\delta(\mu')$, $\nu \in \cB_\delta(\nu')$, and $\Winfty(\mu',\nu') \leq \eps$. By the gluing lemma, there exists a joint distribution for $(X,X',Y,Y')$ with $X \sim \mu, X' \sim \mu', Y \sim \nu, Y \sim \nu'$ such that $\PP(X \neq X'), \PP(Y \neq Y') \leq \delta$ and $|X'-Y'| \leq \eps$. By a union bound, $|X-Y| \leq \eps$ with probability at least $1 - 2\delta$.

On the other hand, if there exists a coupling $\pi = (1 - \delta)\pi_g + \delta \pi_b \in \Pi(\mu,\nu)$ for $\pi_g,\pi_b \in \cP(\cX \times \cX)$ such that $\|d\|_{L^\infty(\pi_g)} \leq \eps$ (i.e., restating the converse assumption), then taking marginals gives $\mu \in \cB_\delta((\pi_g)_1)$ and $\nu \in \cB_\delta((\pi_g)_2)$. Hence, $\pi_g$ is feasible for the (near) coupling problem \cref{prop:RWp-coupling-primal}, and $\Winfty^\delta(\mu,\nu) \leq \eps$.
\end{proof}

We can use this fact to define a Robust Wasserstein Mechanism that achieves $(\eps,\delta)$-PP.

\begin{algorithm}[b!]
\centering
    \caption{Robust Wasserstein Mechanism $(f, \{\cS, \mathcal{Q}, \Theta \}, \eps, \delta$)}
    \begin{algorithmic}
        \State \textbf{for} all $(S,T)\in \mathcal{Q}$ and $\theta\in \Theta$ such that $\PP(X_\theta \in S) > 0$ and $\PP(X_\theta \in T) > 0$
        \State \indent Set $\mu_{|S}^{\theta}\coloneqq \cL\big(f(X_\theta) \big|S\big)$ and $\mu_{|T}^{\theta}\coloneqq \cL\big(f(X_\theta) \big|T\big)$
        \State Calculate $W_\delta=\sup_{\substack{(S,T)\in \mathcal{Q}, \theta\in \Theta\\\PP(X_\theta \in S) > 0, \PP(X_\theta \in T) > 0}}\Winfty^\delta \big(\mu_{|S}^{\theta}, \mu_{|T}^{\theta}\big)$.
        \State Sample $Z\sim \mathrm{Lap}(W_\delta/\eps)$.
        \State Return $\mathsf{M}(X_\theta)=f(X_\theta)+Z$
    \end{algorithmic}
    \label{alg.wass}
\end{algorithm}

\begin{theorem}[Robust Wasserstein Mechanism]
\label{thm:pp}
For any PP framework $(\cS,\cQ,\Theta)$ and $\delta \leq 1/2$, the Robust Wasserstein Mechanism from Algorithm \ref{alg.wass} is $(\eps,2\delta)$-PP. %
\end{theorem}

\begin{proof}
The proof follows that of the original Wasserstein mechanism established in \cite{song2017}. We will write $p_X$ to denote the probability density function of an absolutely continuous random variable $X$.
Let $\mathsf{M}$ denote the robust Wasserstein mechanism
and consider the distributions $\mu_{|S}^{\theta}, \mu_{|T}^{\theta}$ associated with a secret pair $(S,T) \in \cQ$ such that $\PP(X_\theta \in S)$ and $\PP(X_\theta \in T)$ are positive. Because $\RWp(\mu_{|S}^{\theta},\mu_{|T}^{\theta}) \leq W_\delta$, we know that there exists a coupling $\pi \in \Pi(\mu_{|S}^{\theta},\mu_{|T}^{\theta})$ as guaranteed by \cref{lem:coupling}. We decompose $\pi = (1 - 2\delta)\pi_g + 2\delta \pi_b$ into pieces so that $\pi_g \in \cP(\cX \times \cX)$ has the $W_\delta$ distance guarantee and $\pi_b$ is arbitrary. Then, for any measurable $Y\subseteq\cY$, we have
\begin{align*}
    p_{\,\mathsf{M}(X_\theta) | S}(w) &= \int p_{f(X_\theta) | S}(a) \, p_Z(w-a) \dd a\\
    &= \int p_{f(X_\theta) | S}(a) e^{-\eps|w-a|/W_\delta} \dd a\\
    &= \int e^{-\eps|w-a|/W_\delta} \dd \pi(a,b)\\
    &\leq (1 - 2\delta)\int e^{-\eps|w-a|/W_\delta} \dd \pi_g(a,b) + 2\delta,
\end{align*}
where we have simply used the definition of the mechanism and the fact that $e^{-\eps|w-a|/W_\delta} \leq 1$. Continuing, we compute
\begin{align*}
    p_{\,\mathsf{M}(X_\theta)|S}(w) &\leq \int_{\{(a,b):|a-b| \leq W_\delta\}}  e^{-\eps|w-a|/W_\delta} \dd \pi_g(a,b) + 2\delta\\
    &\leq e^{\eps} \int_{\{(a,b):|a-b| \leq W_\delta\}}  e^{-\eps|w-b|/W_\delta} \dd \pi_g(a,b) + 2\delta\\
    &\leq e^{\eps} \int e^{-\eps|w-b|/W_\delta} \dd \pi(a,b) + 2\delta\\
    &\leq e^{\eps} \int  e^{-\eps|w-b|/W_\delta} \, p_{f(X_\theta)|T}(b) \dd b + 2 \delta\\
    &\leq e^{\eps} p_{\,\mathsf{M}(X_\theta)|T}(w) + 2\delta,
\end{align*}
using the distance property of the coupling, performing a change of variables, and using the definition of the mechanism. Since the previous argument holds for any secret pair, we conclude that the mechanism is $(\eps,2\delta)$-PP private.
\end{proof}

This result suggests that $\Winfty^{\delta}$ naturally generalizes the relation between $\Winfty$ and $(\eps,0)$-PP to the $(\eps,\delta)$-PP regime. In applications where a positive $\delta$ is tolerable, the Robust Wasserstein Mechanism allows a significantly lower scale of noise relative to the original Wasserstein Mechanism from \cite{song2017} --- i.e., $W_\delta\ll W$ for some Pufferfish frameworks --- corresponding to a substantial improvement in accuracy. To demonstrate this, consider the following toy example.

\begin{example} Consider a toy scenario where a bank has 100 customers of two types, $A$ and $B$, with at most 10 of type $B$. The incomes of Type $A$ and Type $B$ customers are distributed according to $\mu_A$ and $\mu_B$, respectively. Suppose that incomes are independent conditioned on the customer types, and the bank wishes to publish the combined income of its customers while keeping the count of $B$ users private. The corresponding PP framework has $\cS$ as the set of possible number of type $B$ customers (between 1 and 10), $\cQ$ contains all distinct pairs from $\cS$, and each $\theta \in \Theta$ gives the joint distribution over incomes given a selection of types. 

In order to publish the sum of incomes while guaranteeing $(\eps,\delta)$-PP, the Robust Wasserstein Mechanism dictates that it suffices to add Laplace noise at scale $\Winfty^{\delta}(\mu_A^{*100},\mu_A^{*90}*\mu_B^{*10})/\eps$, where $\gamma^{*k}$ denotes a $k$-fold convolution of $\gamma$ with itself. Assume all incomes are less than \$100k, except Type $B$ customers who have probability $\delta'$ of being millionaires, with $1 - (1 - \delta')^{10} \ll \delta$. In this case, the classic Wasserstein Mechanism would require noise at scale $\sim 10^8/\eps$, while the $\delta$-robust mechanism only needs noise at scale $\sim 10^7/\eps$. Thus, utilizing $\Winfty^\delta$ and its connection to $(\eps,\delta)$-PP allows injecting an order of magnitude less noise.
\end{example}

\begin{remark}[$(\eps,\delta)$-PP and $\Wp$]
Using the robust Wasserstein distance we can further connect PP and classic $\Wp$. This may be practically advantageous, as finite-order Wasserstein distances are often easier to compute than $\Winfty$. First note that Markov's inequality implies that %
$\RWinfty(\mu,\nu) = O\big(\Wp(\mu,\nu)/\delta^{1/p}\big)$. Hence, the mechanism that adds noise at scale $W_p/(\eps \delta^{1/p})$, where $W_p$ is a uniform bound on $\Wp(\mu_{|S}^{\theta},\mu_{|T}^{\theta})$, is $(\eps,\delta)$-Pufferfish private.
\end{remark}

\end{appendices}

\end{document}